\documentclass[11pt]{article}
\usepackage[margin=1in]{geometry}
\usepackage{times}
\usepackage{amsmath,amssymb,amsthm}
\usepackage{makecell}
\usepackage{tikz}
\usepackage{floatpag}
\usepackage{multirow}
\usepackage{titletoc}
\usepackage[ruled,vlined,linesnumbered]{algorithm2e}
\DontPrintSemicolon
\usepackage[authoryear,round]{natbib}
\usepackage[colorlinks=true,citecolor=green!40!black,linkcolor=black,urlcolor=black]{hyperref}
\usepackage{aliascnt}
\usepackage[nameinlink,noabbrev]{cleveref}
\crefname{assumption}{Assumption}{Assumptions}
\crefname{appendix}{Appendix}{Appendices}
\crefname{definition}{Definition}{Definitions}
\crefname{equation}{Equation}{Equations}
\crefname{figure}{Figure}{Figures}
\crefname{algorithm}{Algorithm}{Algorithms}
\crefname{corollary}{Corollary}{Corollaries}
\crefname{lemma}{Lemma}{Lemmas}
\crefname{section}{Section}{Sections}
\crefname{table}{Table}{Tables}
\crefname{theorem}{Theorem}{Theorems}
\newcommand{\E}{\mathbb E}
\newcommand{\R}{\mathbb R}
\newcommand{\Reg}{\mathrm{Reg}}
\newcommand{\poly}{\mathrm{poly}}
\newcommand{\tildep}{\widetilde{p}}

\newcommand{\tildeM}{\widetilde{M}}

\newcommand{\conv}{\mathrm{conv}}
\newcommand{\calA}{\mathcal{A}}

\newcommand{\calE}{\mathcal{E}}

\newcommand{\calI}{\mathcal{I}}
\newcommand{\calJ}{\mathcal{J}}
\newcommand{\calK}{\mathcal{K}}

\newcommand{\calN}{\mathcal{N}}
\newcommand{\calO}{\mathcal{O}}

\newcommand{\calS}{\mathcal{S}}

\newcommand{\inner}[1]{\langle #1\rangle}
\newcommand{\Inner}[1]{\left\langle #1\right\rangle}
\newcommand{\norm}[1]{\lVert #1\rVert}

\newcommand{\sfL}{\mathsf{L}}
\newcommand{\thetahat}{\widehat{\theta}}
\newcommand{\Thetahat}{\widehat{\Theta}}
\newcommand{\numberthis}{\refstepcounter{equation}\tag{\theequation}}
\renewcommand{\d}{\mathrm{d}}
\DeclareMathOperator*{\argmax}{arg\,max}
\DeclareMathOperator{\Cov}{Cov}
\DeclareMathOperator{\tr}{tr}
\DeclareMathOperator{\op}{op}
\newtheorem{lemma}{Lemma}
\newaliascnt{theorem}{lemma}
\newtheorem{theorem}[theorem]{Theorem}
\aliascntresetthe{theorem}
\newaliascnt{corollary}{lemma}
\newtheorem{corollary}[corollary]{Corollary}
\aliascntresetthe{corollary}
\theoremstyle{plain}
\newtheorem{definition}{Definition}
\theoremstyle{plain}
\newtheorem{assumption}{Assumption}
\allowdisplaybreaks[2]
\usepackage{microtype}
\title{Follow the Perturbed Leader for Linear Bandits: \\$\sqrt{T}$ Regret over General Action Sets}
\author{Chen-Yu Wei\\
University of Virginia\\
\texttt{chenyu.wei@virginia.edu}}
\date{}
\begin{document}
\maketitle
\begin{abstract}
We give the first adversarial linear bandit algorithm based on Follow-the-Perturbed-Leader (FTPL) that achieves $\sqrt{T}$ regret over general action sets. Our algorithm attains $\widetilde{\calO}(d^{5/4}\sqrt{T})$ regret using only linear optimization oracles, and the near-optimal $\widetilde{\calO}(d\sqrt{T})$ regret when an optimal-design exploration distribution is further available. This improves on the $T^{2/3}$ regret of the FTPL algorithm of \cite{dani2006} for general action sets, and generalizes the $\sqrt{T}$ regret of \cite{levy2026} over two special action sets (hypercube and Euclidean ball). Compared to previous oracle-efficient linear bandit algorithms based on continuous exponential weights, the FTPL approach generally has better computational efficiency and is much simpler to implement.  
 
The key algorithmic element is a perturbation distribution that adapts to the past reward estimators. Its simplest form, an adaptive Gaussian perturbation, is straightforward to implement and already attains the regret bound above. We further extend it to two general constructions, \emph{random-weighting} and \emph{elliptical} perturbations, each covering distributions from light tails (Gaussian) to heavy tails (Student-$t_2$). On the technical side, we develop a systematic way to analyze the stability term in the FTPL regret through the Fisher information of the perturbation distribution. This builds a general framework for FTPL beyond existing ones such as the hazard rate and the FTRL-FTPL duality. 
\end{abstract}

\section{Introduction}\label{sec:introduction}

In a $d$-dimensional linear bandit problem, in each round, the environment chooses a reward vector $\theta_t\in \R^d$ and the learner chooses an action $a_t$ from an action set $\calA\subset \R^d$. Then the learner receives the scalar reward $\inner{a_t, \theta_t}$ as feedback. Linear bandits model decision-making problems where the number of possible actions could be enormous, but each action is encoded as a low-dimensional feature vector and the reward depends on the action only through the feature. 
 
The adversarial linear bandit problem considers the case where $\theta_t$ could be time-varying, which makes the formulation robust to changing environments. The minimax regret of adversarial linear bandits is of order $d\sqrt{T}$ up to logarithmic factors \citep{dani2008price}, and there are polynomial-time algorithms achieving this rate. However, these algorithms are in general highly computationally demanding or restrictive: the SCRiBLe algorithm \citep{abernethy2009competing} based on Follow-the-Regularized-Leader (FTRL) requires access to a self-concordant barrier for the action set (including the ability to compute its gradient and Hessian); algorithms based on continuous exponential weights \citep{hazan2016volumetric, ito2019oracle} can be run on general convex sets as long as the learner is equipped with a membership or separation oracle for the action set, but the time and oracle complexities of these algorithms usually involve high powers of $d$ and $T$ with large leading constants, making them generally impractical.

On the other hand, Follow-the-Perturbed-Leader (FTPL) methods are known for their computational efficiency, requiring only a linear optimization oracle. They have been the method of choice for combinatorial bandits with \emph{semi-bandit} feedback, where the reward of each selected component is revealed \citep{neu2016importance}. With \emph{full-bandit} feedback, where only the total reward is revealed, combinatorial bandits become a special case of linear bandits, and the best known rate with FTPL is $T^{2/3}$ in general \citep{awerbuch2004, dani2006}. Recently, \citet{levy2026} developed an FTPL algorithm achieving $\widetilde{\calO}(d\sqrt{T})$ regret for the Euclidean ball and hypercube. However, these are highly structured action sets for which efficient specialized algorithms are available. For general combinatorial action sets, where the computational advantages of FTPL are particularly valuable, achieving $\sqrt{T}$ regret with FTPL had remained open.

In this work, we develop the first FTPL algorithm for linear bandits achieving $\sqrt{T}$ regret over general action sets (\cref{thm:main-informal}). Among oracle-efficient algorithms for general action sets, it arguably offers the best balance between regret and computation, and it is simple to implement. In particular, for non-convex action sets (e.g., combinatorial problems), our FTPL operates directly on $\calA$ without additional overhead, since a linear optimization oracle for $\calA$ also serves a linear optimization oracle for $\conv(\calA)$. In contrast, existing continuous exponential weights algorithms based on membership \citep{hazan2016volumetric} or separation \citep{ito2019oracle} oracles require additional procedures to implement these oracles using linear optimization, and to decompose points in $\conv(\calA)$ into convex combinations of actions in $\calA$. Both introduce overhead.

\begin{table}[t]
\centering
\footnotesize
\renewcommand{\arraystretch}{1.5}
\setlength{\tabcolsep}{4pt}
\caption{Comparison of FTPL algorithms for adversarial linear bandits. $R$ is an outer radius of the action set after normalizing the covariance of the exploration basis; see \cref{ass:action-geometry} for its definition. Obtaining the optimal design exploration basis may be computationally hard for general action sets, but it can be efficiently constructed in special cases such as the hypercube, Euclidean ball, $m$-set, probability simplex, and the space of permutation matrices.}
\vspace{5pt}
\label{tab:regret-comparison}
\begin{tabular}{|l|l|l|l|l|}
\hline
\textbf{Method} & \textbf{Action set} & \textbf{Regret in $\widetilde{\calO}$} & \textbf{Regret notion} & \textbf{Exploration basis} \\
\hline
\citet{awerbuch2004}
& general
& $d^{5/3}T^{2/3}$
& pseudo regret & barycentric spanner \\
\hline
\makecell[l]{\cite{mcmahan2004}\\[-1pt]\cite{dani2006}} 
& general
& $dT^{2/3}$
& expected regret & barycentric spanner \\
\hline
\citet{levy2026}
& \makecell[l]{hypercube or \\[-1pt] Euclidean ball}
& $d\sqrt{T}$
& pseudo regret & -- \\
\hline
\textbf{This paper}
& general
& $(d+Rd^{3/4})\sqrt{T}$
& pseudo regret & \makecell[l]{barycentric spanner ($R\leq 4\sqrt{d}$)\\[-1pt]optimal design ($R=1$)}\\
\hline
\end{tabular}
\par\medskip

\end{table}

The key element underlying our algorithm is a time-varying, data-adaptive perturbation distribution. In its simplest form, the perturbation at round $t$ is drawn from $\mathcal{N}\big(0, \lambda I + \sum_{\tau<t}\widehat{\theta}_\tau\widehat{\theta}_\tau^\top\big)$, where $\widehat{\theta}_\tau$ are the reward estimators and $\lambda>0$. This instance is easy to implement and already attains our main result:

\begin{theorem}[Main result, informal]\label{thm:main-informal}
Let $\calA\subset\R^d$ be an action set accessed through a linear optimization oracle. FTPL with the adaptive Gaussian perturbation above, combined with geometric resampling and barycentric-spanner exploration, achieves $\Reg_T(a)=\widetilde{\calO}(d^{5/4}\sqrt{T})$ for every $a\in\calA$. With an optimal-design exploration distribution instead, the regret improves to the near-optimal $\widetilde{\calO}(d\sqrt{T})$. The same bounds hold for every perturbation distribution in \cref{sec: algorithm overview}. See \cref{cor:final-multiplier-bounds} for the formal statement.
\end{theorem}

In \cref{sec: motivation}, we show that no fixed perturbation distribution can control the stability term through the action covariance, as the analyses of FTRL approaches such as SCRiBLe and continuous exponential weights do. This motivates the use of adaptive perturbations. In a similar spirit, \cite{kim2019optimality} show that no perturbation in multi-armed bandits recovers FTRL with the Tsallis entropy regularizer, and \cite{abernethy2014smoothing} argue that while FTPL can be mapped to an FTRL, there is no bijection between them. These arguments all concern fixed perturbations and do not apply to adaptive ones, so the adaptive perturbation scheme in this work can bypass these restrictions.
 
Beyond Gaussian, we develop a fairly general theoretical framework for FTPL for linear bandits. First, we propose two adaptive perturbation constructions, random-weighting perturbations and elliptical perturbations, both containing the Gaussian perturbation as a special case. The latter gives a standard perturbation scheme where a single perturbation decides the policy in every round, while the former requires redrawing the weights of all previous rounds, but is conceptually interesting and admits a simpler analysis. Second, under each of these two constructions, we instantiate the algorithm with a wide range of perturbation distributions capturing different tail behaviors: from light-tailed Gaussian and Laplace distributions to heavy-tailed multivariate Student-$t_2$ distributions. Third, we analyze the stability of the FTPL update through the Fisher information of the perturbation and the stability of the Fisher information matrix under shifting. The analysis bears similarity to the local-Hessian-norm stability in typical SCRiBLe analysis. We hope that the analysis in this work could help develop a systematic framework for analyzing FTPL.
 
We provide a comparison with prior FTPL algorithms for linear bandits in \cref{tab:regret-comparison}. We also compare the oracle and computational complexity of our algorithm with those of previous oracle-efficient algorithms based on FTRL and log-concave sampling in \cref{tab:computational-comparison}. The FTPL approach remains more favorable even when prior approaches are implemented with state-of-the-art log-concave sampling techniques. Closest to our approach, adaptive Gaussian perturbations have recently been used for full-information online learning \citep{buzaglo2026oracle, voitovych2026oracle, jiang2026adaptive}; we study the bandit setting and also cover non-Gaussian perturbations. Further related work on FTPL is discussed in \cref{sec:other-related-work}.

\section{Notation and Setting}\label{sec:notation}
\paragraph{Notation.} For a probability distribution $p$ on $\R^d$ with finite second moment,
define its mean by $\mu(p)=\E_{a\sim p}[a]$ and its covariance matrix by $\Cov(p)
 =\E_{a\sim p}\bigl[(a-\mu(p))(a-\mu(p))^\top\bigr]$. For two functions $f$ and $g$, the notation ``$f(x)\lesssim g(x)$'' indicates that there exists a universal constant $c>0$ such that $f(x)\leq cg(x)$ for all $x$. This is equivalent to writing $f(x) = \calO(g(x))$, and $\widetilde{\calO}$ further hides polylogarithmic factors.  $\Delta_{\calS}$ denotes the space of distribution over $\calS$. 
 
\paragraph{Setting.}
The learner has an action set $\calA\subset\R^d$ and access
to a linear optimization oracle over $\calA$. At each round $t$, an
adversary chooses an affine reward function
\begin{align*}
 r_t(a)=\inner{a,\theta_t}+b_t,
 \qquad |r_t(a)|\le1\quad\text{for every }a\in\calA.
\end{align*}
We allow the intercept $b_t$ because the normalization below (\cref{thm:normalization}) may translate the action set, which turns a linear reward on the original set into an affine one. 
The adversary may depend on the past, but chooses $\theta_t$ and $b_t$ before
the learner's random draws at round $t$. The learner then samples
and plays $a_t\in\calA$ and observes only $y_t=r_t(a_t)$. Let $\E_t$ denote expectation conditional on the history before round $t$.
For every fixed comparator $a\in\calA$, we study the regret: 
\begin{align*}
 \Reg_T(a)=\E\left[\sum_{t=1}^T\bigl(r_t(a)-r_t(a_t)\bigr)\right]
 =\E\left[\sum_{t=1}^T\inner{a-a_t,\theta_t}\right].
\end{align*}

\subsection{Action Set Normalization}
We work in the affine hull of the action set and let $d\ge1$ denote its
dimension. Throughout the analysis, we impose the following geometric
and exploration assumptions.
 
\begin{assumption}[Action set and exploration basis]\label{ass:action-geometry}
The action set $\calA\subset\R^d$ and a fixed exploration
distribution $\nu\in\Delta_\calA$ satisfy
\begin{align*}
 rB_2^d\subseteq\conv(\calA)\subseteq RB_2^d,
 \qquad r>0,
 \qquad \mu(\nu)=0,\qquad\Cov(\nu)=I/d.
\end{align*}
Here $B_2^d$ is the Euclidean unit ball.
\end{assumption}
 
The following theorem gives two ways to obtain coordinates satisfying \cref{ass:action-geometry}. Detailed statements and proofs are given in \cref{thm:spanner-normalization,thm:optimal-design-normalization} in \cref{app:spanner-normalization}.
 
\begin{theorem}[Action set normalization]\label{thm:normalization}
Let $\calK\subset\R^d$ be a compact set such that $\conv(\calK)$ has nonempty interior. Suppose that $\calK$ is accessed through a linear optimization oracle. There exist an affine transformation $\calA$ of $\calK$ and an exploration distribution $\nu$ on $\calA$ satisfying \cref{ass:action-geometry} with
\begin{itemize}
    \item[(i)] $r=1/d$ and $R=4\sqrt{d}$, computable from the barycentric spanner of \citet{awerbuch2004} using $\calO(d^2\log(d+1))$ oracle calls and $\calO(d^4\log(d+1))$ additional arithmetic operations; or
    \item[(ii)] $r=1/d$ and $R=1$, given by an optimal design $q^\star\in\argmax_{q\in\Delta_\calK}\det\Cov(q)$.
\end{itemize}
\end{theorem}
 
The second option attains the optimal outer radius $R=1$; its optimal-design distribution is also called John's exploration distribution \citep{bubeck2012towards}. Computing an optimal design is equivalent to computing the minimum-volume ellipsoid enclosing $\calK$ \citep{todd2016minimum}, which is NP-hard even for polytopes given by linear inequalities \citep{boyd2004convex}, and we are not aware of an efficient algorithm that finds a constant-factor approximation using only a linear optimization oracle. However, in some special cases, such a design has a closed form or is easy to construct, e.g., the probability simplex, Euclidean ball, hypercube, $m$-set, the set of $m$-by-$m$ permutation matrices, and others.

\section{Adaptive Perturbation: Motivation and Central Ideas}\label{sec: motivation}
 
We first attempt to use a fixed perturbation distribution and identify an obstacle (\cref{sec: fixed pert}). Then we show how adaptive perturbation addresses it (\cref{sec: adaptive pert}). 

\subsection{Fixed Perturbation and Obstacles}\label{sec: fixed pert}
Assume for now that we use a fixed perturbation distribution~$\rho$. 
Suppose that at round $t$, the cumulative reward estimator is $\Thetahat_t = \sum_{\tau=1}^{t-1} \thetahat_\tau$ where $\thetahat_\tau$ is the standard reward estimator for linear bandits in round~$\tau$. Then FTPL generates action $a_t$ as:  
\begin{align*}
    \text{Sample } Z_t\sim \rho \text{ and let } a_t = x(\Thetahat_t + Z_t), \qquad \text{where } x(\Theta):=\argmax_{a\in\calA} \inner{a, \Theta}.  
\end{align*}
Write $h(\Theta)=\max_{a\in\calA}\inner{a,\Theta}$ and define
\begin{align*}
 \Psi(\Theta)=\E_{Z\sim\rho}[h(\Theta+Z)].   \numberthis \label{eq: potential def}
\end{align*}
For every perturbation density considered in this paper, $\Psi$ is twice continuously differentiable and $\nabla\Psi(\Theta)=\E_{Z\sim\rho}[x(\Theta+Z)]$ (\cref{lem:smoothed-support-hessian}), so $\E_t[a_t]=\nabla\Psi(\Thetahat_t)$.
The key step in the regret analysis is bounding the \emph{stability term}: 
\begin{align*}
    \Psi(\Thetahat_{t+1}) - \Psi(\Thetahat_{t}) - \inner{\nabla \Psi(\Thetahat_t), \thetahat_t} 
    &= 
    \Psi(\Thetahat_{t} + \thetahat_t) - \Psi(\Thetahat_{t}) - \inner{\nabla \Psi(\Thetahat_t), \thetahat_t}  \\ 
    &= 
    \frac12\thetahat_t^\top \nabla^2 \Psi(\Thetahat_t + s\thetahat_t) \thetahat_t \numberthis \label{eq: shifting}
\end{align*}
for some $s\in[0,1]$ by Taylor's theorem. Ignoring the shift $s\thetahat_t$, we can approximate the stability term as $\frac12\thetahat_t^\top \nabla^2 \Psi(\Thetahat_t) \thetahat_t$.  
Motivated by the standard FTRL analysis (SCRiBLe or continuous exponential weights), an ideal property we would hope for is 
\begin{align*}
    \nabla^2 \Psi(\Thetahat_t) \preceq c M_t := c \int  \big(x(\Thetahat_t + z) -\mu_t\big)\big(x(\Thetahat_t + z) -\mu_t\big)^\top \rho(z) \d z    \numberthis \label{eq: desired bound}
\end{align*}  
for some constant $c>0$ independent of $\Thetahat_t$,  
where $\mu_t = \E_{Z\sim \rho}[x(\Thetahat_t+Z)]$ is the action mean induced by FTPL, and $M_t$ is the action covariance. 
If \cref{eq: desired bound} holds, the expectation of the stability term can be bounded by $c\E[\thetahat_t^\top M_t \thetahat_t] \lesssim cd$, as in FTRL.
 
\paragraph{Can \cref{eq: desired bound} hold? } For multi-armed bandits ($\calA=\{e_1, e_2, \ldots, e_d\}$), \citet{abernethy2015fighting} characterize the stability of FTPL through the \emph{hazard rate} of the perturbation and give several perturbations whose hazard rate is bounded by a constant. In \cref{thm:hazard-covariance} (\cref{app:hazard-covariance}), we show that i.i.d.\ coordinate-wise perturbations with hazard rate bounded by $h_q$ satisfy \cref{eq: desired bound} on the simplex with $c=2h_q$. For \emph{general convex sets}, however, the answer is negative. The theorem below constructs an action set whose boundary contains exposed edges at arbitrarily small scales. No fixed perturbation can satisfy \cref{eq: desired bound} with a finite constant on this set.
\begin{theorem}\label{thm: impossible}
Define $\calA
 =\operatorname{conv}\left(\{\pm e_1,\pm e_2\}
       \cup\{v_i:i\ge2\}\right)\subset\R^2$ with $v_i = (\cos(2^{-i}\pi), \sin(2^{-i}\pi))$, as illustrated in \cref{fig:multiscale-action-set}. 
For every regular perturbation density $\rho$ on $\R^2$ (\cref{lem:smoothed-support-hessian}), define $M(\Theta)
 =\Cov_{Z\sim\rho}\bigl(x(\Theta+Z)\bigr)$. 
Then for every finite
$c>0$, there exists $\Theta\in\R^2$ such that
\begin{align*}
 \nabla^2\Psi(\Theta)\npreceq cM(\Theta).
\end{align*}
\end{theorem}

\begin{figure}[t]
\centering
\begin{tikzpicture}[scale=2.75, every node/.style={font=\small}]
  \coordinate (ep) at (1,0);
  \coordinate (em) at (-1,0);
  \coordinate (up) at (0,1);
  \coordinate (down) at (0,-1);
  \coordinate (v2) at ({cos(45)},{sin(45)});
  \coordinate (v3) at ({cos(22.5)},{sin(22.5)});
  \coordinate (v4) at ({cos(11.25)},{sin(11.25)});
  \coordinate (v5) at ({cos(5.625)},{sin(5.625)});
  \coordinate (v6) at ({cos(2.8125)},{sin(2.8125)});
 
  \draw[gray!35] (0,0) circle (1);
  \draw[->,gray!65] (-1.15,0)--(1.18,0);
  \draw[->,gray!65] (0,-1.15)--(0,1.18);
  \filldraw[fill=blue!8,draw=blue!65!black,thick]
    (ep)--(v6)--(v5)--(v4)--(v3)--(v2)--(up)--(em)--(down)--cycle;
  \draw[red!75!black,very thick] (v3)--(v4)
    node[midway,left] {$E_3$};
 
  \foreach \p in {ep,em,up,down,v2,v3,v4,v5,v6}
    \fill[blue!65!black] (\p) circle (0.018);
  \node[above right] at (v2) {$v_2$};
  \node[above right] at (v3) {$v_3$};
  \node[right] at (v4) {$v_4$};
  \node[below right] at (ep) {$e_1$};
  \node[above right] at (up) {$e_2$};
  \node[above left] at (em) {$-e_1$};
  \node[below right] at (down) {$-e_2$};
\end{tikzpicture}
\caption{The action set $\calA$ in \cref{thm: impossible}. Its exposed edges
$[v_i,v_{i+1}]$ accumulate at $e_1$, and their lengths converge to zero.}
\label{fig:multiscale-action-set}
\end{figure}
The proof is given in \cref{app:hazard-covariance}. In \cref{fig:multiscale-action-set}, when the endpoints $v_i$ and $v_{i+1}$ of the edge $E_i=[v_i,v_{i+1}]$ are the two competing actions (i.e., the current near-optimal actions), we show that the stability term scales with $|E_i|$ while the action covariance scales with $|E_i|^2$, so if there exists a $c$ satisfying \cref{eq: desired bound}, it has to scale with $1/|E_i|$, which goes to infinity when $i$ increases. On a polytope, the same argument shows that $c$ must scale at least inversely with its shortest edge length.
 
This provides a contrast between FTPL and FTRL  (SCRiBLe and continuous exponential weights). In FTRL, the stability term can be bounded by a constant multiple of $\thetahat_t^\top M_t \thetahat_t = \E_{a\sim p_t}[\inner{a - \mu_t, \thetahat_t}^2]$, where $p_t$ is the action distribution used in round $t$. This scales \emph{quadratically} in both the reward estimator $\thetahat_t$ and the action separation $a-\mu_t$. In FTPL, the construction in \cref{thm: impossible} indicates that the stability term, while scaling \emph{quadratically} with the reward estimator, only scales \emph{linearly} with the action separation.  Indeed, \cref{lem:smoothed-support-hessian} gives
\begin{align*}
    \thetahat_t^\top \nabla^2 \Psi(\Thetahat_t) \thetahat_t
    &= -\int \inner{x(\Thetahat_t + z)-\mu_t, \thetahat_t} \inner{\nabla \rho(z), \thetahat_t} \d z \\
    &= -\E_{Z\sim \rho} \left[ \inner{x(\Thetahat_t + Z) - \mu_t, \thetahat_t} \inner{\nabla \log \rho(Z), \thetahat_t} \right],  \numberthis  \label{eq: FTPL analysis}
\end{align*}
which scales only linearly with the action separation. In fact, \cref{eq: FTPL analysis} would match the FTRL form if $\nabla \log \rho(Z) = -x(\Thetahat_t + Z)+\mu_t$, but this is not possible with a fixed perturbation $\rho$. FTPL gains this factor back in the other \emph{penalty term} in the regret, which scales linearly with the action separation, while that of FTRL has no scaling with it. In multi-armed bandits, the balance between penalty and stability can be achieved with a global learning rate tuning \citep{honda2023best}. However, in linear bandits with general action sets such as the one in \cref{thm: impossible}, the tuning has to be adaptive to the scale of the action separation between competitive actions, which depends on the local structure of the action set. Therefore, a more geometry-aware tuning is needed. This is handled in the next subsection.
 
\subsection{Adaptive Perturbation}\label{sec: adaptive pert} 
We continue from \cref{eq: FTPL analysis} with Cauchy--Schwarz:
\begin{align*}
      \thetahat_t^\top \nabla^2 \Psi(\Thetahat_t) \thetahat_t 
      &\leq \sqrt{\E_{Z\sim \rho} \left[\inner{x(\Thetahat_t + Z) - \mu_t, \thetahat_t}^2\right] } \sqrt{\E_{Z\sim \rho} \left[\inner{\nabla \log \rho(Z), \thetahat_t}^2 \right]}  \\ 
      &= \sqrt{\thetahat_t^\top M_t \thetahat_t} \sqrt{\thetahat_t^\top \calI(\rho)\thetahat_t}
\end{align*}
where $\calI(\rho):=\E_{Z\sim \rho}[\nabla \log\rho(Z)\nabla \log\rho(Z)^\top]$ is the Fisher information matrix of the perturbation distribution $\rho$. Since $\E_t[\thetahat_t^\top M_t\thetahat_t]\lesssim d$, applying Cauchy--Schwarz again over rounds shows that it remains to bound $\sum_{t=1}^T \thetahat_t^\top \calI(\rho)\thetahat_t$ of order $\poly(d,\log T)$. This suggests an adaptive perturbation (drawing $Z_t\sim \rho_t$) such that 
\begin{align*}
   \calI(\rho_t) \approx \Sigma_t^{-1}, \qquad \quad \Sigma_t = \lambda I + \sum_{\tau=1}^{t-1} \thetahat_\tau\thetahat_\tau^\top \text{ for some } \lambda >0,  
\end{align*}
since the elliptical potential lemma then gives $\sum_{t=1}^T \thetahat_t^\top \calI(\rho_t)\thetahat_t \lesssim d\log T$ for large enough $\lambda$, giving a total stability of order $d\sqrt{T\log T}$. A natural choice is $\rho_t=\calN(0,\Sigma_t)$, whose Fisher information matrix is exactly $\Sigma_t^{-1}$. Intuitively, since $\thetahat_\tau$ scales with the inverse of the action covariance, $\Sigma_t$ grows fastest in directions where competing actions are close together, which is where \cref{thm: impossible} calls for a larger perturbation.
 
We will show that this simple adaptive Gaussian perturbation indeed achieves the bound in \cref{thm:main-informal} after additionally handling the Hessian shift in \cref{eq: shifting} and the change of the perturbation across rounds, and combining it with a standard reward estimator (geometric resampling) and exploration (barycentric spanner or optimal design). Besides the Gaussian, whose tail decays as $e^{-x^2}$, we extend the results, with the same regret bound, to Laplace perturbations (tail $e^{-x}$) and multivariate Student-$t_2$ elliptical perturbations (tail $x^{-2}$), demonstrating the generality of the framework. These heavier-tailed distributions might find applications in settings that also require differential privacy.
 

\section{Algorithm Overview}\label{sec: algorithm overview}

\begin{figure}[p]
\thisfloatpagestyle{empty}
\begin{minipage}{\linewidth}
\begin{algorithm}[H]
\caption{FTPL (with spanner exploration and geometric resampling)}
\label{alg:ftpl}
\SetKwProg{Fn}{Procedure}{}{}
\SetKwFunction{SampleAction}{SampleAction}
\SetKwFunction{GeometricResampling}{GeometricResampling}
\textbf{Inputs:} Action set $\calA$ with a linear optimization oracle,
 exploration distribution $\nu\in\Delta(\calA)$,
 exploration rate $\gamma\in[0,1]$, resampling step size $\beta>0$,
 truncation $N\ge0$, and a construction of the perturbation density $\rho_t$.\;
\textbf{Initialization:} $\Thetahat_1\gets0$.\;
\For{$t=1,2,\ldots$}{
Construct $\rho_t$ from $\{\thetahat_\tau\}_{\tau<t}$ using one of the perturbation schemes in \cref{fig:perturbation-choices}.\;
Set $a_t\gets$ \SampleAction{$\Thetahat_t,\rho_t$}, play $a_t$, and observe $y_t=\inner{a_t, \theta_t} + b_t$.\;
$\bar a_t\gets$ \SampleAction{$\Thetahat_t,\rho_t$}.\;
$H_t\gets$ \GeometricResampling{$\Thetahat_t,\rho_t,\beta,N$}.\;
$\displaystyle \thetahat_t\gets H_t(a_t-\bar a_t)y_t$.\;
$\displaystyle \Thetahat_{t+1}\gets\Thetahat_t+\thetahat_t$.\;
}
\BlankLine
\Fn{\SampleAction{$\Theta,\rho$}}{
Draw $e\sim\operatorname{Bernoulli}(\gamma)$.\;
\eIf{$e=1$}{
Draw an action $a\sim\nu$ and \KwRet $a$.\;
}{
Draw a perturbation $Z\sim\rho$.\;
\KwRet an action in $\argmax_{a\in\calA}\inner{a,\Theta+Z}$,
 using a fixed tie-breaking rule.\;
}
}
\BlankLine
\Fn{\GeometricResampling{$\Theta,\rho,\beta,N$}}{
\For{$i=1,\ldots,N$}{
$u_i\gets$ \SampleAction{$\Theta,\rho$} and $v_i\gets$ \SampleAction{$\Theta,\rho$}.\;
$Q_i\gets I-\frac{\beta}{2}(u_i-v_i)(u_i-v_i)^\top$.\;
}
\KwRet $\beta I+\beta\sum_{k=1}^N Q_1\cdots Q_k$.\;
}
\end{algorithm}
\end{minipage}
 
\par\smallskip
\centering
\caption{Two adaptive perturbation constructions for \cref{alg:ftpl}.}
\label{fig:perturbation-choices}
\setlength{\fboxsep}{8pt}
\medskip
\fbox{\begin{minipage}{\dimexpr\linewidth-2\fboxsep-2\fboxrule\relax}
\begin{center}
\textbf{Option 1: Random-Weighting Perturbations}
\end{center}
\vspace{-10pt}
Choose distribution $\rho_1$ from \cref{tab:convolution-initializers}, with parameter $\lambda>0$. \\
Choose distribution $\varphi$ from \cref{tab:scalar-noise}. \\ 
Define $Z\sim \rho_t$ to be equivalent to independently drawing $Z'\sim \rho_1$ and $\xi_{\tau}\sim\varphi$ for all $\tau<t$, and setting
\begin{align*}
 Z=Z'+\sum_{\tau=1}^{t-1}\xi_{\tau}\thetahat_\tau. 
\end{align*}
Equivalently, $\rho_t(z)=\E_{\xi_1\sim \varphi,\ \cdots,\ \xi_{t-1}\sim \varphi }\left[\rho_1\left(z - \sum_{\tau=1}^{t-1}\xi_\tau \thetahat_\tau\right)\right]$.
\end{minipage}} 
 
\par\medskip
\fbox{\begin{minipage}{\dimexpr\linewidth-2\fboxsep-2\fboxrule\relax}
\begin{center}
   \textbf{Option 2: Elliptical Perturbations}
\end{center}
\vspace{-10pt}
Choose distribution family $P_\Sigma$ from \cref{tab:elliptical-updates}. \\ 
In round $t$, define
\begin{align*}
 \Sigma_t=\lambda I+\sum_{\tau=1}^{t-1}\thetahat_\tau\thetahat_\tau^\top,
 \qquad \rho_t=P_{\Sigma_t}. 
\end{align*}
\end{minipage}}
\end{figure}

Our algorithm is displayed in \cref{alg:ftpl}. In each round $t$, the learner samples a perturbation $Z_t\in \R^d$ from the density $\rho_t$, and takes action $a_t=\argmax_{a\in\calA}\inner{a, \Thetahat_t + Z_t}$ with probability $1-\gamma$, and action $a_t\sim \nu$ with probability $\gamma$, where $\nu\in\Delta_\calA$ is the exploration distribution based on spanner or optimal design (\cref{ass:action-geometry}). In order to construct the reward estimator $\thetahat_t$ at time $t$, the learner follows the same procedure to draw $2N+1$ additional actions internally (one as $\bar a_t$ and the other $2N$ as in the Geometric Resampling procedure). With them, the algorithm constructs $H_t$ as an estimator of the inverse action covariance matrix $\tildeM_t^{-1}$ via matrix geometric resampling \citep{neu2020efficient}, where $\tildeM_t:=\Cov((1-\gamma)p_t+\gamma\nu)$ and $p_t$ is the law of $x(\Thetahat_t+Z)$ for $Z\sim\rho_t$, and uses it to construct the reward estimator $\thetahat_t$.\footnote{A more commonly seen reward estimator is $\thetahat_t = H_t'a_ty_t$ where $H_t'$ estimates the \emph{inverse second moment} of the action distribution $\E[a_ta_t^\top]^{-1}$ instead of the inverse covariance. That handles \emph{linear reward} but not the more general \emph{affine reward} we consider.} The simplest perturbation construction, which already achieves the bound in \cref{thm:main-informal}, is the following:
\begin{align*}
     Z_t \sim \mathcal{N}\left(0, \lambda I + \sum_{\tau=1}^{t-1} \thetahat_\tau \thetahat_\tau^\top \right),   \numberthis \label{eq: Guassian pert} 
\end{align*}
with $\lambda=\widetilde\Theta(d^{1/2}T)$. We enrich the family of possible perturbations by proposing two general constructions of $\rho_t$ (\cref{fig:perturbation-choices}), both including the Gaussian perturbation \cref{eq: Guassian pert} as a special case and achieving the same regret bound.

\paragraph{Random-Weighting Perturbations.}
In round $t$, independently draw $Z_t'\sim\rho_1$ and weights $\xi_{t,\tau}\sim\varphi$ for every $\tau<t$, and set $Z_t=Z_t'+\sum_{\tau=1}^{t-1}\xi_{t,\tau}\thetahat_\tau$. Equivalently,
\begin{align*}
    a_t = \argmax_{a\in\calA}\Inner{a,\sum_{\tau=1}^{t-1} (1+\xi_{t,\tau}) \thetahat_\tau + Z_t' }.     \numberthis \label{eq: choice of a_t}
\end{align*}
The factor $(1+\xi_{t,\tau})$ may be viewed as a reweighting of the round-$\tau$ sample in round $t$. 
 
In the regret analysis, it is crucial that $Z_t'$ and $\{\xi_{t,\tau}\}_{\tau<t}$ are \emph{fresh} samples in every round. 
Therefore, the random-weighting perturbation is, in general, not as computationally favorable as standard FTPL, as forming the weighted sum requires $\calO(t)$ time. However, when $\varphi$ is the standard Gaussian, the sampling procedure is equivalent to drawing $Z_t'\sim\rho_1$ and $\Xi_t\sim \mathcal{N}(0, \sum_{\tau=1}^{t-1} \thetahat_\tau\thetahat_\tau^\top)$ independently, and setting $Z_t=Z_t'+\Xi_t$, which admits a sampling cost independent of $t$. 
 
We provide possible choices of the initializer $\rho_1$ and weight distribution $\varphi$ in \cref{tab:convolution-initializers,tab:scalar-noise}, respectively. At the optimized scale (\cref{eq:final-scale-choices}), $\lambda=\widetilde\Theta(\sqrt d\,T)$ for every combination. Both can be chosen from a wide range of distributions and can be arbitrarily combined.
 
This decision rule has intriguing resemblance to perturbed-history exploration \citep{kveton2019perturbed},  the garbage-in-reward-out (GIRO) strategy \citep{kveton2019garbage}, and FTPL with dropout perturbation \citep{van2014follow}, where data samples in each of the past rounds are individually reweighted. In \cite{kveton2019garbage} and \cite{van2014follow}, each past data sample is either \emph{chosen} or \emph{unchosen}, corresponding to our case when $\xi_{t,\tau}\in\{-1,1\}$. However, our analysis requires that $\varphi$ be a smooth distribution, ruling out such binary-valued perturbations. The connection between these works is an interesting future direction. We remark that \cite{kveton2019garbage,kveton2019perturbed} study the stochastic bandit setting and \cite{van2014follow} study the adversarial full-information setting, while we study the strictly more challenging adversarial bandit setting.

\paragraph{Elliptical Perturbations.}  
We generalize the Gaussian perturbation in another manner. We consider perturbation distributions whose density $\rho_t(z)$ is proportional to $f(\|z\|_{\Sigma_t^{-1}})$, where $\Sigma_t = \lambda I + \sum_{\tau<t} \thetahat_\tau\thetahat_\tau^\top$, for some radial profile $f: \R_{\geq 0}\to \R_{\geq 0}$. The adaptive Gaussian perturbation (\cref{eq: Guassian pert}) is a special case when $f(s)\propto e^{-s^2/2}$.  Sampling $Z_t\sim \rho_t$ can be decomposed into the following steps: first, randomly sample a direction $U$ uniformly on the unit sphere $\mathbb S^{d-1}$; second, sample the radius $S$ from a density proportional to $s^{d-1}f(s)$; finally, let $Z_t=\Sigma_t^{1/2}SU$. Different choices of $f$ are listed in \cref{tab:elliptical-updates}. 
 
Besides the Gaussian, our list also includes the radial Laplace and the multivariate Student-$t_2$ distributions. These distributions have appeared in prior work on differential privacy \citep{hardt2010geometry, abernethy2019online, bun2019average} to achieve pure differential privacy. To our knowledge, these applications either do not concern online learning or only address the full-information setting. How our algorithm connects with differential privacy is left for future work.   
 
Random-weighting and elliptical perturbations do not include each other. Their only intersection is Gaussian. All of perturbations we consider achieve the same regret bound. Their relative merits are left for future investigation.
 
In \cref{sec:decomposition}, we provide a regret decomposition into stability terms and bias terms. In \cref{sec:convolved} and \cref{sec:elliptical} we bound the stability terms for random-weighting perturbations and elliptical perturbations, respectively. Note that the time-varying perturbation introduces an additional stability term, beyond the one analyzed in \cref{sec: motivation}, due to the change of the perturbation.  In \cref{sec:estimator}, we bound the bias due to Geometric Resampling, and combine all regret terms.

\section{Regret Decomposition}\label{sec:decomposition}
As before, we denote $h(\Theta)=\max_{a\in \calA}\inner{a,\Theta}$ and $ x(\Theta)=\argmax_{a\in \calA}\inner{a,\Theta}$, 
where we break ties in $x(\Theta)$ in an arbitrary but fixed way.
By the definition, $h$ is
convex and $\nabla h(\Theta)=x(\Theta)$ almost everywhere. Define the potential function by
\begin{align*}
 \Psi_t(\Theta)=\E_{Z\sim\rho_t}[h(\Theta+Z)].
\end{align*} 
Let $p_t$ be the
law of $x(\Thetahat_t+Z)$ for $Z\sim\rho_t$.  
Define the action mean and action covariance under $p_t$ by
\begin{align*}
 \mu_t=\mu(p_t)=\E_{Z\sim \rho_t}[x(\Thetahat_t+Z)], \qquad M_t= \Cov(p_t) = \E_{a\sim p_t}[(a-\mu_t)(a-\mu_t)^\top]. 
\end{align*}
We decompose the regret of the sequence $\mu_t$ into the following: 
\begin{align*}
 \sum_{t=1}^T\inner{a-\mu_t,\thetahat_t}
 =\inner{a,\Thetahat_{T+1}} + \sum_{t=1}^T(A_t+B_t)+\Psi_1(0)-\Psi_{T+1}(\Thetahat_{T+1}).
\end{align*}
where  
\begin{align}
 A_t&=\Psi_t(\Thetahat_{t+1})-\Psi_t(\Thetahat_t)
                  -\inner{\mu_t,\thetahat_t},\label{eq:A}\\
 B_t&=\Psi_{t+1}(\Thetahat_{t+1})-\Psi_t(\Thetahat_{t+1}).\label{eq:B}
\end{align}
$A_t$ is the standard stability term as discussed in \cref{sec: motivation}. The additional stability term $B_t$ comes from the changes of the perturbation over time. 
For any fixed comparator $a\in \calA$, the terminal potential satisfies 
$\Psi_{T+1}(\Thetahat_{T+1})
 \ge\inner{a,\Thetahat_{T+1}}+\inner{a,\E_{Z\sim\rho_{T+1}}[Z]}$ by its definition. Substitution proves
\begin{equation}\label{eq:decomposition}
 \sum_{t=1}^T\inner{a-\mu_t,\thetahat_t}
 \le \Psi_1(0)-\inner{a,\E_{Z\sim\rho_{T+1}}[Z]}+\sum_{t=1}^T(A_t+B_t).
\end{equation}
 
Recall that $\E_t$ conditions on the history before round $t$, and let $\tildep_t = (1-\gamma)p_t + \gamma \nu$ be the distribution from which the learner samples $a_t$. Under this conditioning, 
$\rho_t,p_t,\tildep_t$ are fixed, while $a_t$ and $\thetahat_t$ are random. 
To convert~\cref{eq:decomposition} into regret, suppose the conditional
estimation bias satisfies
\begin{align*}
 \norm{\E_t[\thetahat_t]-\theta_t}_2\le\varepsilon_t,
\end{align*}
where $\varepsilon_t\ge0$ is determined by the past. 
Then, 
\begin{align*}
 \E_t[\inner{a-a_t,\theta_t}]
 ={}&\E_t[\inner{a-\mu_t,\thetahat_t}]
 +\inner{a-\mu_t,\theta_t-\E_t[\thetahat_t]}
 +\gamma\inner{\mu_t-\mu(\nu),\theta_t}\\
 \le{}&\E_t[\inner{a-\mu_t,\thetahat_t}]
       +2R\varepsilon_t+2\gamma.
\end{align*}
Summing and applying
\cref{eq:decomposition} gives
\begin{equation}\label{eq:general-regret}
 \begin{aligned}
 \Reg_T(a)\le{}&\E\left[\Psi_1(0)
 -\inner{a,\E_{Z\sim\rho_{T+1}}[Z]}
 +\sum_{t=1}^T(A_t+B_t)\right]+2R\sum_{t=1}^T\E[\varepsilon_t]+2\gamma T.
 \end{aligned}
\end{equation}
In the following sections, we bound the individual terms above.

\section{Stability Analysis for Random-Weighting Perturbations}\label{sec:general}\label{sec:convolved}
We bound $A_t$ and $B_t$ for the random-weighting perturbation construction in
\cref{fig:perturbation-choices}. 
 
\subsection{Bounding \texorpdfstring{$A_t$}{A-t}}\label{sec:A-bound}
Bounding $A_t$ follows the same idea in \cref{sec: adaptive pert}, except that we have to handle the shifting in the Hessian (\cref{eq: shifting}). The Taylor remainder evaluates the Hessian at the shifted point $\Thetahat_t+s\thetahat_t$, while the estimator only controls the action covariance $M_t$ at $\Thetahat_t$. Translating the integral moves the shift from the action side to the density side, where it appears as the ratio $\rho(z)/\rho(z+v)$, which is close to one when $v$ is small relative to the perturbation scale. This leads to the following definition of shifted Fisher-information matrix: 
\begin{definition}[Shifted Fisher-information matrix]\label{def:fisher-information}
Let $\rho$ be a probability density on $\R^d$ that is positive almost everywhere.
Define the shifted Fisher-information
matrix by
\begin{align}
 \calI^v(\rho)&=\int_{\R^d}
   \frac{\nabla\rho(z)\nabla\rho(z)^\top}{\rho(z+v)}\,\d z = \int_{\R^d}
   \frac{\rho(z)^2}{\rho(z+v)}\nabla\log\rho(z)\nabla\log\rho(z)^\top \d z,
   \qquad v\in\R^d.
   \label{eq:shifted-fisher-information}
\end{align}
For $v=0$, write $\calI(\rho):=\calI^0(\rho)$. 
At round $t$, write
$\calI_t^v:=\calI^v(\rho_t)$ and $\calI_t:=\calI(\rho_t)$.
\end{definition}
We can upper bound $A_t$ through the shifted Fisher information matrix. The theorem below holds for both random-weighting and elliptical perturbations. 
\begin{theorem}[Bounding $A_t$ with $\calI_t^v$]\label{thm:A-bound}
We have 
\begin{equation}\label{eq:A-bound}
 A_t \le\frac12\sqrt{\thetahat_t^\top M_t\thetahat_t}
       \sup_{s\in[0,1]}\sqrt{\thetahat_t^\top\calI_t^{s\thetahat_t}\thetahat_t}.
\end{equation}
\end{theorem}

With \cref{eq:A-bound}, the key is to bound $\calI_t^{s\thetahat_t} \preceq \text{constant}\cdot \Sigma_t^{-1}$ for all $s\in[0,1]$, where $\Sigma_t = \lambda I + \sum_{\tau=1}^{t-1}\thetahat_\tau \thetahat_\tau^\top$. The following \cref{lem:convolved-shifted} establishes this bound, with the constant depending on the stability of the initial distribution $\rho_1$ under shifting and the Fisher information of the weight distribution $\varphi$. The bounds for the distributions we use are summarized in \cref{tab:convolution-initializers,tab:scalar-noise} and proved in \cref{app:scalar-table,app:elliptical-update-energy}.
 
\begin{lemma}\label{lem:convolved-shifted}
Consider random-reweighting perturbations. Recall that $\calI_t^v$ denotes the shifted Fisher-information matrix defined in \cref{eq:shifted-fisher-information}.
Fix any $\lambda>0$ and set
$\Sigma_t=\lambda I+\sum_{j=1}^{t-1}\thetahat_j\thetahat_j^\top$.
Then, for every shift $v\in\R^d$,
\begin{equation}\label{eq:convolved-general-derivative}
 \calI_t^v
 \preceq\max\left\{
   \lambda\norm{\calI_1^v}_{\op},\,
   \bigl(1+\chi^2(\rho_1(\cdot-v)\|\rho_1)\bigr)I_{\varphi}
 \right\}\Sigma_t^{-1},
\end{equation}
where $I_{\varphi}=\int\frac{\varphi'(\xi)^2}{\varphi(\xi)}\,\d \xi$ is the Fisher information associated with the scalar density $\varphi$, and $\rho_1(\cdot-v)$ is the density of $Z+v$ for $Z\sim\rho_1$.
\end{lemma}
At $v=0$, \cref{lem:convolved-shifted} is a matrix analogue of the Fisher information inequality for convolutions \citep{liu2006two}. The shift $v$ contributes the additional $\chi^2$ factor.
Combining \cref{lem:convolved-shifted} with upper bounds on $\lambda\|\calI_1^{s\thetahat_t}\|_{\op}$ and $1+\chi^2(\rho_1(\cdot - s\thetahat_t)\|\rho_1)$ (\cref{tab:convolution-initializers}) and $I_\varphi$ (\cref{tab:scalar-noise}) for individual distributions, we can derive bounds for $A_t$: 
\begin{corollary}
\label{cor:A-multiplier}
Suppose $\lambda$ is chosen such that 
$\norm{\thetahat_t}_2\le\sqrt\lambda/8$ almost surely. Then for every combination of $\rho_1$ in \cref{tab:convolution-initializers} and $\varphi$ in
\cref{tab:scalar-noise}, it holds that 
\begin{align*}
   A_t \lesssim \sqrt{\thetahat_t^\top M_t\thetahat_t}\sqrt{\thetahat_t^\top \Sigma_t^{-1}\thetahat_t}. 
\end{align*}
\end{corollary}

\subsection{Bounding \texorpdfstring{$B_t$}{B-t}}\label{sec:B-convolution}
By the nature of random-weighting perturbations, the stability term due to the change of perturbation is similar to the stability due to inclusion of a new sample, so bounding $B_t$ follows a similar proof as bounding $A_t$. A difference is that the weight $\xi_{t,\tau}$ is unbounded (\cref{eq: choice of a_t}), so in the analysis we perform clipping on its tail. We establish the following theorem (\textit{cf. } \cref{thm:A-bound}). 
\begin{theorem}[Bounding $B_t$ with $\calI_t^v$ under random-weighting perturbations]\label{thm:B-fisher}
Under the random-weighting setup in \cref{fig:perturbation-choices},
for every threshold $\sfL\ge1$, we have
\begin{equation}\label{eq:B-convolution-tail}
 \begin{aligned}
 B_t  &\leq \frac12\sqrt{\thetahat_t^\top M_t\thetahat_t}\,
 \E_{\xi\sim \varphi}\left[\xi^2\mathbf1_{\{|\xi|\le \sfL\}}\sup_{s\in[0,1]}
       \sqrt{\thetahat_t^\top\calI_t^{(1+s\xi)\thetahat_t}\thetahat_t}\right]+2R\norm{\thetahat_t}_2\,\E_{\xi\sim \varphi}[|\xi|\mathbf1_{\{|\xi|>\sfL\}}].
 \end{aligned}
\end{equation}
\end{theorem}
The second term in \cref{eq:B-convolution-tail} handles the rare event when the magnitude of the perturbation $|\xi|$ exceeds a certain threshold. Bounds for this tail term are listed in \cref{tab:scalar-noise} for each perturbation we consider. The corollary below provides a concrete bound for each perturbation:
\begin{corollary}
\label{cor:B-multiplier}
Suppose $\norm{\thetahat_t}_2\le G$ almost surely and $\lambda\ge C G^2\log^2(RGT)$ for a sufficiently large universal constant $C$.  Then for every combination of $\rho_1$ in \cref{tab:convolution-initializers} and $\varphi$ in
\cref{tab:scalar-noise}, it holds that 
\begin{align*}
    B_t \lesssim \sqrt{\thetahat_t^\top M_t\thetahat_t}\sqrt{\thetahat_t^\top \Sigma_t^{-1}\thetahat_t} + T^{-1}. 
\end{align*}
 
\end{corollary}

\section{Stability Analysis for Elliptical Perturbations}\label{sec:elliptical}
We now bound $A_t$ and $B_t$ for the elliptical perturbation construction in
\cref{fig:perturbation-choices}.

\subsection{Bounding \texorpdfstring{$A_t$}{A-t}}\label{sec:A-elliptical}
Under elliptical perturbations, $A_t$ is again bounded by \cref{thm:A-bound}. The bounds in \cref{tab:elliptical-updates} give $\calI^v(P_\Sigma)\preceq \calO(1)\Sigma^{-1}$ whenever $\norm v_{\Sigma^{-1}}\le1/8$. Applying this to $v=s\thetahat_t$ for every $s\in[0,1]$ yields the following corollary:
\begin{corollary}
\label{cor:A-elliptical}
Suppose $\lambda$ is chosen such that 
$\norm{\thetahat_t}_2\le\sqrt\lambda/8$ almost surely. Then for every perturbation in \cref{tab:elliptical-updates}, it holds that
\begin{align*}
   A_t\lesssim \sqrt{\thetahat_t^\top\Sigma_t^{-1}\thetahat_t}
 \sqrt{\thetahat_t^\top M_t\thetahat_t}
\end{align*}
\end{corollary}
 
\subsection{Bounding \texorpdfstring{$B_t$}{B-t}}\label{sec:B-elliptical}
Unlike the random-weighting construction, $\rho_{t+1}=P_{\Sigma_{t+1}}$ is not obtained from $\rho_t$ by adding independent noise (except in the Gaussian case). To express the bound on $B_t$, we define a different information matrix, called \emph{shifted elliptical information matrix} below. Since it involves the second moment of $Z$, which is infinite for Student-$t_2$, we also need a clipped version. 
 
\begin{definition}[Shifted elliptical information matrix]\label{def:elliptical-update-information}
Fix an elliptical family $P_\Sigma$, a shape matrix $\Sigma\succ0$, and
$v\in\R^d$. Define
\begin{equation}\label{eq:elliptical-update-information}
 \calJ^v(P_\Sigma)
 :=\int_{\R^d}
 \frac{P_{\Sigma+vv^\top}(z)^2}{P_\Sigma(z+v)}
 \Sigma^{-1}zz^\top\Sigma^{-1}\,\d z.
\end{equation}
For a clipping radius $\sfL>0$, define the ellipsoid $\calE_\Sigma(\sfL)
 :=\left\{z\in\R^d:
 \norm z_{\Sigma^{-1}}\le \sfL\right\}$ 
and the clipped version of $\calJ^v$ as 
\begin{equation}\label{eq:clipped-elliptical-update-information}
 \calJ^{v,\sfL}(P_\Sigma)
 :=\int_{\calE_\Sigma(\sfL)}
 \frac{P_{\Sigma+vv^\top}(z)^2}{P_\Sigma(z+v)}
 \Sigma^{-1}zz^\top\Sigma^{-1}\,\d z.
\end{equation}
 
For $v=0$, write $\calJ(P_\Sigma):=\calJ^0(P_\Sigma)$ and
$\calJ^{\sfL}(P_\Sigma):=\calJ^{0,\sfL}(P_\Sigma)$. For the
elliptical law at round $t$, write $\calJ_t^v:=\calJ^v(P_{\Sigma_t})$
and $\calJ_t:=\calJ(P_{\Sigma_t})$, and similarly write
$\calJ_t^{v,\sfL}:=\calJ^{v,\sfL}(P_{\Sigma_t})$ and
$\calJ_t^{\sfL}:=\calJ^{\sfL}(P_{\Sigma_t})$.
\end{definition}
Observe that if $P_\Sigma$ is Gaussian, then $\calJ(P_{\Sigma})=\calI(P_\Sigma)$ because $\nabla P_{\Sigma}(z) = -P_\Sigma(z)\Sigma^{-1}z$. The next theorem bounds $B_t$ through elliptical information matrix.

\begin{samepage}
\begin{theorem}[Bounding $B_t$ with $\calJ^v_t$ under elliptical perturbation]\label{thm:B-elliptical}
Under the elliptical perturbation setup in \cref{fig:perturbation-choices},
assume the radial profile $f$ is positive and nonincreasing. Then 
\begin{equation}\label{eq:B-elliptical-unclipped}
 B_t\le\frac12
 \sqrt{\thetahat_t^\top M_t\thetahat_t}
 \sqrt{\thetahat_t^\top\calJ_t^{\thetahat_t}\thetahat_t}.
\end{equation}
Write $\kappa_t=\norm{\thetahat_t}_{\Sigma_t^{-1}}$.   For every finite clipping radius $\sfL\ge\sqrt{1+\kappa_t^2}$,
\begin{equation}\label{eq:B-elliptical}
 \begin{aligned}
 B_t\le{}&\frac12
 \sqrt{\thetahat_t^\top M_t\thetahat_t}
 \sqrt{\thetahat_t^\top\calJ_t^{\thetahat_t,\sfL}\thetahat_t}+\frac{2R\norm{\thetahat_t}_2\kappa_t}{\sqrt d}
 \E_{Z\sim P_{\Sigma_t}}\!\left[
   \norm Z_{\Sigma_t^{-1}}
   \mathbf1\left\{\norm Z_{\Sigma_t^{-1}}>\frac{\sfL}{\sqrt{1+\kappa_t^2}}\right\}
 \right].
 \end{aligned}
\end{equation}
\end{theorem}
\end{samepage}
\begin{corollary}
\label{cor:B-elliptical}
Suppose $\lambda$ is chosen such that 
 $\norm{\thetahat_t}_2\le\sqrt\lambda/8$ almost surely. Then for every perturbation in \cref{tab:elliptical-updates}, it holds that
\begin{align*}
    B_t \lesssim \sqrt{\log(dRGT)}\sqrt{\thetahat_t^\top M_t\thetahat_t}
 \sqrt{\thetahat_t^\top \Sigma_t^{-1}\thetahat_t} + T^{-1}. 
\end{align*}
\end{corollary}
 

\section{Bias Analysis for Geometric Resampling and Final Regret Bounds}\label{sec:estimator}
 
\subsection{Geometric Resampling}
Geometric resampling \citep{neu2016importance, neu2020efficient} has been a standard procedure for constructing reward estimators for FTPL in bandits. The next theorem summarizes its guarantee. The bounds in \cref{thm:estimator-bounds} play the following roles: the bias $\varepsilon_t$ enters \cref{eq:general-regret} directly, the range bound for $\|\thetahat_t\|_2$ determines how large $\lambda$ must be, which further affects $\Phi_1(0)$, and the bound on $\E_t[\thetahat_t^\top M_t \thetahat_t]$ controls a factor in the stability terms.
\begin{theorem}[Geometric-resampling estimator bounds]\label{thm:estimator-bounds}
In the normalized setting, let $\gamma\in(0,1)$, $N\ge0$ be an integer,
and $0<\beta\le1/(2R^2)$. For the geometric-resampling estimator
$\thetahat_t=H_t(a_t-\bar a_t)y_t$ in \cref{alg:ftpl}, set
$\varepsilon_t=\norm{\E_t[\thetahat_t]-\theta_t}_2$. Then 
\begin{equation}\label{eq:estimator-bounds}
 \begin{aligned}
 \varepsilon_t\le\frac{e^{-\beta\gamma(N+1)/d}}{r},\qquad 
 \norm{\thetahat_t}_2 \le2R\beta(N+1), \qquad 
 \E_t[\thetahat_t^\top M_t\thetahat_t]&\le\frac{4d}{1-\gamma}.
 \end{aligned}
\end{equation}
\end{theorem}

\subsection{Final Regret Bounds and Parameter Choices}
\label{sec:final-bounds}
 
We combine stability terms and bias terms in \cref{eq:general-regret} to establish the final regret bound.
Choose the following hyperparameters for \cref{alg:ftpl}: 
\begin{align}
 L&=\log(4RT/r),\qquad \beta=\frac{1}{2R^2},\qquad 
 \gamma=\frac{4R(L+1)d^{3/4}}{\sqrt T},\qquad
 N+1=\left\lceil\frac{Ld}{\beta\gamma}\right\rceil.
 \label{eq:parameters}
\end{align}
Assume $\gamma\leq 1/2$ without loss of generality. Define
$G=2R\beta(N+1)$. Since $R^2\ge1$ by the exploration covariance,
\cref{thm:estimator-bounds} and \cref{eq:parameters} give
\begin{equation}\label{eq:parameter-bounds}
 \begin{aligned}
 \max_{t\le T}\norm{\thetahat_t}_2
 \le G\le\frac{2Rd(L+1)}{\gamma}
 =\frac{d^{1/4}\sqrt T}{2}, \qquad 
 \varepsilon_t \le\frac{1}{4RT},\qquad
 \E_t[\thetahat_t^\top M_t\thetahat_t]\le8d.
 \end{aligned}
\end{equation}
Choose the initial scale as
\begin{equation}\label{eq:final-scale-choices}
 \lambda=C G^2
 \begin{cases}
  \log^2(RGT),&\text{for random-weighting perturbations},\\
  1,&\text{for elliptical perturbations},
 \end{cases}
\end{equation}
where $C$ is a sufficiently large universal constant. These choices give
$\lambda=\widetilde\Theta(\sqrt d\,T)$ and
$N=\Theta(Rd^{1/4}\sqrt T)$.
 
The stability terms $\sum_t (A_t+B_t)$ contribute $\widetilde\calO(d\sqrt T)$ through the corollaries in \cref{sec:convolved,sec:elliptical} and the elliptical potential lemma, both of which require $\sqrt\lambda\gtrsim\max_t\norm{\thetahat_t}_2$ (\cref{eq:final-scale-choices}). Since geometric resampling only guarantees $\max_t\norm{\thetahat_t}_2\lesssim Rd/\gamma$, the penalty $\Psi_1(0)\lesssim R\sqrt{d\lambda}$ (\cref{lem:initial-potential}) is of order $R^2d^{3/2}/\gamma$, while exploration costs $\gamma T$. Balancing the two gives the $Rd^{3/4}\sqrt T$ term in \cref{eq:final-regret-bound}. This is formally established in the following theorem. 
 
\begin{theorem}
\label{cor:final-multiplier-bounds}
Under \cref{ass:action-geometry}, use the parameters above. For random-weighting perturbations with any initializer from \cref{tab:convolution-initializers}
and any scalar law from \cref{tab:scalar-noise}, or for elliptical perturbations with any radial law from \cref{tab:elliptical-updates}, \cref{alg:ftpl} achieves for every $a\in\calA$,
\begin{equation}\label{eq:final-regret-bound}
 \Reg_T(a)=\widetilde{\calO}\big((d+Rd^{3/4})\sqrt T\big).
\end{equation}
\end{theorem}
The proof is given in \cref{app:final-proof}. Since $R$ enters \cref{eq:final-regret-bound} only through $Rd^{3/4}$, the barycentric spanner ($R\le4\sqrt d$) gives $\widetilde{\calO}(d^{5/4}\sqrt T)$ and the optimal design ($R=1$) gives the near-optimal $\widetilde{\calO}(d\sqrt T)$. More generally, any exploration distribution with $R=\calO(d^{1/4})$ already yields $\widetilde{\calO}(d\sqrt T)$.

\section{Open Problems}

There are several open questions. First, the computational bottleneck of our algorithm lies in the geometric resampling procedure (\cref{app:computational-comparison}), making the total time complexity scale with $T^{3/2}$, which remains superlinear in time. It remains open whether there exists an oracle-efficient linear bandit algorithm for general action sets achieving $\poly(d)\sqrt{T}$ regret with the computational complexity scaling linearly with $T$. Second, it remains open whether FTPL can achieve first-order, second-order \citep{rakhlin2013online,bubeck2019improved, ito2020tight}, or best-of-both-worlds regret \citep{kong2023best, ito2023best, dann2023blackbox}. The immediate challenge towards these results is the explicit spanner exploration we utilize in our algorithm, which leads to $\sqrt{T}$ regret at best.  Third, it remains unclear how to achieve high-probability regret or corruption robustness with FTPL. Existing results usually rely on sophisticated biasing techniques specific to SCRiBLe and continuous exponential weights \citep{lee2020bias, zimmert2022return, liu2024corruption}. Finally, it would be interesting to extend FTPL to linear contextual bandits \citep{neu2020efficient, liu2023bypassing, van2026improved} and linear MDPs \citep{luo2021policy,sherman2023improved, dai2023refined, kong2023improved, liu2024towards}, where the available actions depend on the context or state. This will allow the techniques to be used in more general and realistic settings. 

\section*{AI Disclosure}
The author used GPT-6 Astra and Claude Opus 5.5 throughout the research and writing. The idea of using an adaptive Gaussian perturbation was proposed by GPT-6 Astra. The Fisher-information analysis framework, its extension to random-weighting and elliptical perturbations, and the non-Gaussian instances were developed through subsequent exchanges between the author and the AIs.

\bibliographystyle{apalike}
\bibliography{ftpl_references}

\clearpage
\appendix

\section*{Appendix Contents}
\startcontents[appendix]
\printcontents[appendix]{}{1}{\setcounter{tocdepth}{2}}

\newpage
\section{Other Related Work on FTPL}\label[appendix]{sec:other-related-work}
 
\paragraph{Full-information FTPL.}
\citet{hannan1957approximation} introduces FTPL for repeated games, and \citet{kalai2005efficient} popularize it for online linear optimization. 
\citet{abernethy2014smoothing} interpret FTPL as stochastic smoothing of
the support function, providing a unified framework for FTPL and FTRL. \citet{cohen2015following} study online learning with combinatorial action sets, establishing regret upper and lower bounds whose gap remains open. 
\citet{van2014follow} study dropout perturbations that adapt to observed data, obtaining $\sqrt{T}$ regret in the adversarial expert setting and constant regret in the stochastic setting. \citet{syrgkanis2016efficient} and \citet{dudik2020oracle} use FTPL for oracle-efficient contextual online learning; \citet{haghtalab2022oracle} and \citet{block2022smoothed} use FTPL for online learning with contexts generated by a smoothed adversary. \citet{suggala2020online} study FTPL for online learning with non-convex losses. 
 
Three recent works use adaptive Gaussian perturbations and are closely related to ours. \citet{buzaglo2026oracle} study online classification and 
\citet{voitovych2026oracle} study online learning with a smoothed adversary, using
adaptive Gaussian perturbations of the form
$\sum_{\tau<t}\xi_\tau f(X_\tau)$, where the $\xi_\tau$ are independent standard Gaussians
and $X_\tau$ are the past covariates; this has the same form as our random-weighting construction with Gaussian weights. \citet{jiang2026adaptive} use a similar adaptive Gaussian perturbation for online matrix optimization. These works study the full-information setting, whereas we study the bandit setting; moreover, our constructions also cover non-Gaussian perturbations.

\paragraph{Bandit FTPL.}
For adversarial multi-armed bandits, \citet{abernethy2015fighting} relate
the stability of FTPL to the hazard rate and obtain near-optimal regret for a broad family of perturbations.
\citet{honda2023best}, \citet{lee2024follow}, \citet{chen2025geometric}, and \citet{lee2026adaptive} establish best-of-both-worlds regret bounds for multi-armed bandits.
For combinatorial semi-bandits, \citet{neu2016importance} combine FTPL
with geometric resampling to obtain $\sqrt{T}$ regret, \citet{neu2015first} further obtains
first-order regret bounds, and \citet{neu2014online} consider the case of stochastic action sets. \citet{chen2025note}, \citet{chen2026further}, and \citet{zhan2026follow} further improve the regret bounds for combinatorial semi-bandits.  
\citet{dai2022follow} extend FTPL to tabular MDPs. 
These feedback models reveal the loss of each individual coordinate, while the linear bandits considered in this work reveal only a single aggregate reward, which introduces additional challenges.

 \newpage
\section{Perturbation Distribution Bounds}\label[appendix]{app:perturbation-tables}
This appendix records the uniform bounds used in
\cref{sec:convolved,sec:elliptical}. A checkmark means that the displayed
bound holds with a universal constant, independent of the dimension,
scale, shift, and clipping radius. All shift bounds below assume
$\kappa\le1/8$, and all clipped bounds assume $\sfL\ge1$.
 
\begin{table}[!ht]
\begin{minipage}{\linewidth}
\centering
\small
\renewcommand{\arraystretch}{1.6}
\setlength{\tabcolsep}{5pt}
\caption{Radial initializer bounds for random-weighting perturbations, with
$\kappa=\norm v_2/\sqrt\lambda\le1/8$.
Proofs appear in \cref{app:elliptical-update-energy}. The distributions listed here are essentially those in \cref{tab:elliptical-updates} with $\Sigma=\lambda I$. }
\label{tab:convolution-initializers}
\vspace{3pt}
\begin{tabular}{|c|c|c|c|}
\hline
Distribution & $\rho_1(z)\propto$
& \makecell[c]{$\lambda\norm{\calI_1^v}_{\op}=\calO(1)$? }
& \makecell[c]{$1+\chi^2(\rho_1(\cdot-v)\|\rho_1)=\calO(1)$?} \\
\hline
Gaussian & $e^{-\norm z_2^2/(2\lambda)}$
& $\checkmark$ & $\checkmark$ \\
\hline
Radial Laplace & $e^{-\sqrt d\norm z_2/\sqrt\lambda}$
& $\checkmark$ & $\checkmark$ \\
\hline
Student-$t_2$ & $(1+\norm z_2^2/\lambda)^{-(d+2)/2}$
& $\checkmark$ & $\checkmark$ \\
\hline
\end{tabular}
\end{minipage}
\par\bigskip
\begin{minipage}{\linewidth}
\centering
\small
\renewcommand{\arraystretch}{1.6}
\setlength{\tabcolsep}{5pt}
\caption{Mean-zero scalar increment laws. The tail bound is uniform for
$\sfL\ge1$; no shift parameter is involved. Proofs appear in
\cref{app:scalar-table}.}
\label{tab:scalar-noise}
\vspace{3pt}
\begin{tabular}{|c|c|c|c|c|}
\hline
Distribution & $\varphi(\xi)\propto$ & $I_\varphi=\calO(1)$?
& $\E[\xi^2]=\calO(1)$?
& $\E[|\xi|\mathbf1_{\{|\xi|>\sfL\}}]\le$ \\
\hline
Gaussian & $e^{-\xi^2/2}$ & $\checkmark$ & $\checkmark$
& $\sqrt{2/\pi}\,e^{-\sfL^2/2}$ \\
\hline
Laplace & $e^{-|\xi|}$ & $\checkmark$ & $\checkmark$
& $(\sfL+1)e^{-\sfL}$ \\
\hline
\end{tabular}
\end{minipage}
\end{table}
 
\begin{table}[!ht]
\centering
\footnotesize
\renewcommand{\arraystretch}{1.8}
\setlength{\tabcolsep}{3pt}
\caption{Elliptical perturbation bounds for
$\kappa=\norm v_{\Sigma^{-1}}\le1/8$ and $\sfL\ge1$.
Here $S=\norm Z_{\Sigma^{-1}}$ for $Z\sim P_\Sigma$.
The clipping columns give explicit upper bounds and are only needed for
Student-$t_2$.
The dash marks the only failed bound: the Student-$t_2$ untruncated
update information is infinite, even at zero shift.
Proofs appear in \cref{app:elliptical-update-energy}.}
\label{tab:elliptical-updates}
\vspace{3pt}
\hspace*{-19pt}
\begin{tabular}{|c|c|c|c|c|c|c|}
\hline
Distribution & $P_\Sigma(z)\propto $
& \makecell[c]{$\calI^v(P_\Sigma)\preceq \calO(1)\Sigma^{-1}$?}
& \makecell[c]{$\calJ^v(P_\Sigma)\preceq \calO(1)\Sigma^{-1}$?}
& $\calJ^{v,\sfL}(P_\Sigma)\preceq$ & $\E[S^2\mathbf1_{\{S\le\sfL\}}]\le$ & $\E[S\mathbf1_{\{S>\sfL\}}]\le$ \\
\hline
Gaussian & $e^{-\|z\|_{\Sigma^{-1}}^2/2}$ & $\checkmark$ & $\checkmark$
& \multicolumn{3}{c|}{(not needed)} \\
\hline
\makecell[c]{Radial \\ Laplace} & $e^{-\sqrt d\|z\|_{\Sigma^{-1}}}$
& $\checkmark$ & $\checkmark$
& \multicolumn{3}{c|}{(not needed)} \\
\hline
Student-$t_2$ & $\big(1+\|z\|_{\Sigma^{-1}}^2\big)^{-(d+2)/2}$ & $\checkmark$ & --
& $6\log(e\sfL)\Sigma^{-1}$ & $\frac d2\log(1+\sfL^2)$
& $\frac d{\sqrt{1+\sfL^2}}$ \\
\hline
\end{tabular}
\end{table}
 
\clearpage
\section{Action Set Normalization}\label[appendix]{app:spanner-normalization}
\subsection{Barycentric-spanner normalization}
\begin{theorem}[Barycentric-spanner normalization]\label{thm:spanner-normalization}
Let $\calK\subset\R^d$ be a compact set such that $\conv(\calK)$ has nonempty interior. Assume the learner has access to
a linear optimization oracle for $\calK$. Using $\calO(d^2\log(d+1))$ oracle calls
plus $\calO(d^4\log(d+1))$ additional computations, one can find
points $v_0,\ldots,v_d\in\calK$ such that every
$a\in\calK$ can be written as
\begin{align*}
 a=\sum_{i=0}^d\lambda_i v_i,\qquad
 \sum_{i=0}^d\lambda_i=1,\qquad |\lambda_i|\le2.
\end{align*}
Let $q$ be the uniform distribution over $\{v_0,\ldots,v_d\}$.
Then the transformed action set
$\calA=d^{-\frac{1}{2}}\Cov(q)^{-\frac{1}{2}}(\calK-\mu(q))$ and the uniform distribution $\nu$ over
$\{d^{-\frac{1}{2}}\Cov(q)^{-\frac{1}{2}}(v_i-\mu(q)):0\le i\le d\}$ satisfy \cref{ass:action-geometry} with $r=1/d$ and $R=4\sqrt{d}$.
\end{theorem}
 
\begin{definition}[Barycentric spanner]\label{def:barycentric-spanner}
Let $\calS\subset\R^m$ span $\R^m$, and let $C\ge1$.
A collection of linearly independent vectors $s_1,\ldots,s_m\in\calS$
is a \emph{$C$-approximate barycentric spanner} of $\calS$ if every
$s\in\calS$ can be written as
\begin{align*}
 s=\sum_{i=1}^m\lambda_i s_i,
 \qquad |\lambda_i|\le C\quad\text{for all }i.
\end{align*}
\end{definition}
 
\begin{proof}[Proof of \cref{thm:spanner-normalization}]
Since $\conv(\calK)$ has nonempty interior, the compact lifted set
$\calS=\{(a,1):a\in\calK\}$ spans $\R^{d+1}$. Apply the
$2$-approximate barycentric-spanner construction of~\citet{awerbuch2004}
to $\calS$. Its linear optimization queries can be implemented using
the oracle over $\calK$, because maximizing
$\inner{(a,1),(v,\alpha)}=\inner{a,v}+\alpha$ requires only maximizing
$\inner{a,v}$ over $\calK$.
Proposition~2.4 of~\citet{awerbuch2004} bounds the oracle complexity in
dimension $m$ by $\calO(m^2\log_C m)$ for approximation factor $C>1$.
With $m=d+1$ and $C=2$, this is $\calO(d^2\log(d+1))$ calls.
For the arithmetic bound, the construction initializes a basis with
$m$ column replacements and then makes $\calO(m\log m)$ improving
replacements. At each stage, recomputing the basis inverse costs
$\calO(m^3)$ operations. Its rows specify the linear objectives for the
determinant-ratio tests, and checking the oracle outputs for all columns
costs $\calO(m^2)$ further operations. This gives the stated conservative
bound of $\calO(m^4\log m)=\calO(d^4\log(d+1))$ additional operations.
The resulting linearly independent lifts $(v_i,1)$ give affinely
independent points $v_0,\ldots,v_d$. For every $a\in\calK$,
\begin{align*}
 (a,1)=\sum_{i=0}^d\lambda_i(v_i,1),\qquad |\lambda_i|\le2.
\end{align*}
The last coordinate forces $\sum_i\lambda_i=1$. Write
\begin{align*}
 c=\mu(q)=\frac1{d+1}\sum_{i=0}^d v_i,\qquad
 G=\Cov(q)=\frac1{d+1}\sum_{i=0}^d(v_i-c)(v_i-c)^\top.
\end{align*}
Consequently, centering
preserves the representation: $a-c=\sum_i\lambda_i(v_i-c)$.
 
Affine independence implies $G\succ0$. Set
$u_i=d^{-1/2}G^{-1/2}(v_i-c)$ and let $U$ be the matrix with columns
$u_0,\ldots,u_d$. By construction,
\begin{align*}
 U\mathbf1=0,\qquad UU^\top=\frac{d+1}{d}I,
\end{align*}
where $\mathbf1\in\R^{d+1}$ is the all-ones vector. Thus $\nu$, the
uniform law on these columns, has mean zero and covariance $I/d$.
Moreover, $d(d+1)^{-1}U^\top U$ is the orthogonal projection onto
$\mathbf1^\perp$, so
\begin{align*}
 U^\top U=\frac{d+1}{d}I-\frac1d\mathbf1\mathbf1^\top,
 \qquad \inner{u_i,u_j}=\frac{d+1}{d}\mathbf1_{i=j}-\frac1d.
\end{align*}
The simplex $\operatorname{conv}\{u_0,\ldots,u_d\}$ therefore has facets
$\inner{u_i,u}=-1/d$ and satisfies $\norm{u_i}_2=1$.
Its defining inequalities are $\inner{u_i,u}\ge-1/d$, all of which hold
when $\norm{u}_2\le1/d$. This proves
\begin{align*}
 \frac1d B_2^d
 \subseteq\operatorname{conv}\{u_0,\ldots,u_d\}\subseteq\conv(\calA).
\end{align*}
For the outer inclusion, any $u=d^{-1/2}G^{-1/2}(a-c)\in\calA$ has the
representation $u=\sum_i\lambda_i u_i$ with
$\sum_i\lambda_i=1$ and $|\lambda_i|\le2$. The Gram identity gives
\begin{align*}
 \norm{u}_2^2=\frac{d+1}{d}\sum_{i=0}^d\lambda_i^2-\frac1d
 \le\frac{4(d+1)^2-1}{d}\le\frac{4(d+1)^2}{d}\le16d.
\end{align*}
Hence $\calA\subseteq4\sqrt d B_2^d$, completing the proof.
\end{proof}
 
\subsection{Optimal-design normalization}\label[appendix]{app:optimal-design-normalization}
\begin{theorem}[Optimal-design normalization]\label{thm:optimal-design-normalization}
Let $\calK\subset\R^d$ be a compact set such that $\conv(\calK)$ has nonempty interior,  
and let $\Delta_\calK$ denote the set of probability distributions supported
on $\calK$. There exists an optimal design
$
 q^\star\in\argmax_{q\in\Delta_\calK}\det\Cov(q)
$
such that $\Cov(q^\star)$ is positive definite and
$
 \sup_{a\in\calK}
 (a-\mu(q^\star))^\top\Cov(q^\star)^{-1}(a-\mu(q^\star))
 \le d.
$
Therefore, the transformed action set
$
 \calA=d^{-\frac12}\Cov(q^\star)^{-\frac12}
       (\calK-\mu(q^\star))
$
and the distribution of
$d^{-\frac12}\Cov(q^\star)^{-\frac12}(X-\mu(q^\star))$ for
$X\sim q^\star$ satisfy \cref{ass:action-geometry} with $r=1/d$ and
$R=1$.
\end{theorem}
 
\begin{proof}[Proof of \cref{thm:optimal-design-normalization}]
Since $\calK$ is compact, $\Delta_{\calK}$ is compact
under weak convergence, and
$q\mapsto\det\Cov(q)$ is continuous.
Hence an optimal design $q^\star$ exists. 
The uniform distribution on any $d+1$ affinely independent
points of $\calK$ has positive definite covariance, so
$C:=\Cov(q^\star)\succ0$. Write $\mu=\mu(q^\star)$.
 
For any $a\in\calK$, consider
$q_\varepsilon=(1-\varepsilon)q^\star+\varepsilon\delta_a$.
Its covariance is
\begin{align*}
 \Cov(q_\varepsilon)
 =(1-\varepsilon)C
 +\varepsilon(1-\varepsilon)(a-\mu)(a-\mu)^\top.
\end{align*}
Optimality of $q^\star$ implies
\begin{align*}
 0\ge\left.\frac{\d}{\d\varepsilon}
       \log\det\Cov(q_\varepsilon)\right|_{\varepsilon=0+}
 =(a-\mu)^\top C^{-1}(a-\mu)-d.
\end{align*}
This proves the claimed leverage bound and $\calA\subseteq B_2^d$.
The transformed distribution $\nu$ has mean zero and covariance $I/d$
by construction.
 
For the inner inclusion, fix a unit vector $u$, let $X\sim\nu$, and put
$Y=\inner{u,X}$ and $H=\max_{a\in\calA}\inner{u,a}$.
Then $Y\in[-1,H]$, $\E [Y]=0$, and $\E [Y^2]=1/d$. Consequently,
\begin{align*}
 0\le\E[(Y+1)(H-Y)]=H-\frac1d.
\end{align*}
Thus the support function of $\calA$ is at least $1/d$ in every unit
direction. This implies
$(1/d)B_2^d\subseteq\conv(\calA)$ and completes the proof.
\end{proof}
 
\section{Proofs for Fixed Perturbations}
\label[appendix]{app:hazard-covariance}
The following theorem shows that, on the probability simplex, a bounded hazard rate \citep{abernethy2015fighting} implies \cref{eq: desired bound}.
\begin{theorem}\label{thm:hazard-covariance}
    Let $\calA=\Delta_d$. Suppose that the perturbation $\rho$ is regular and draws i.i.d.
    noise in each coordinate, and the distribution in each coordinate has density $q$ with bounded hazard rate
    $h_q=\sup_{u\in\R}\frac{q(u)}{1-Q(u)}$, where $Q$ is the CDF
    of~$q$. Define
    $M(\Theta)=\Cov_{Z\sim\rho}\bigl(x(\Theta+Z)\bigr)$.
    Then, for every $\Theta\in\R^d$,
    \begin{align*}
        \nabla^2\Psi(\Theta)\preceq2h_qM(\Theta).
    \end{align*}
\end{theorem}
 
\begin{proof}[Proof of \cref{thm:hazard-covariance}]
Fix any $\Theta\in\R^d$ and write $K=\nabla^2\Psi(\Theta)$ and
$\mu=\E_{Z\sim\rho}[x(\Theta+Z)]$. The FTPL action is $e_i$ with
probability
\begin{align*}
 \mu_i=\Pr\!\left(\Theta_i+Z_i>\max_{j\ne i}(\Theta_j+Z_j)\right),
 \qquad M(\Theta)=\operatorname{diag}(\mu)-\mu\mu^\top.
\end{align*}
Let $T_i=\max_{j\ne i}(\Theta_j+Z_j)$. Independence of $Z_i$ and $T_i$
gives
\begin{align*}
 \mu_i&= \Pr(Z_i > T_i - \Theta_i) =\E[1-Q(T_i-\Theta_i)],\\
 K_{ii}=\frac{\partial\mu_i}{\partial\Theta_i}
 &=\E[q(T_i-\Theta_i)]
 \le h_q\E[1-Q(T_i-\Theta_i)]
 =h_q\mu_i.
\end{align*}
Conditioning instead on
$\Theta_i+Z_i$ gives
\begin{align*}
 \mu_i=\int_\R q(z-\Theta_i)
              \prod_{\ell\ne i}Q(z-\Theta_\ell)\,\d z.
\end{align*}
Thus, for $i\ne j$,
\begin{align*}
 K_{ij}=\frac{\partial\mu_i}{\partial\Theta_j}
 =-\int_\R q(z-\Theta_i)q(z-\Theta_j)
       \prod_{\ell\notin\{i,j\}}Q(z-\Theta_\ell)\,\d z\le0.
\end{align*}
Interchanging $i$ and $j$ leaves this integral the same, so
$K_{ij}=K_{ji}$. Moreover, adding the same constant to every coordinate
of $\Theta$ does not change the maximizing action, and hence does not
change $\mu$. This implies that $K \mathbf{1}=0$.
Therefore, with $w_{ij}=-K_{ij}=w_{ji}\ge0$ for $i\ne j$, we have
$K_{ii}=\sum_{j\ne i}w_{ij}$. Fix $v\in\R^d$, and define
$\bar v=\inner{\mu,v}$ and $\widetilde v=v-\bar v\mathbf{1}$. Then
\begin{align*}
 v^\top Kv
 = \widetilde{v}^\top K\widetilde{v} &=\sum_{i<j}w_{ij}(\widetilde v_i-\widetilde v_j)^2\\
 &\le2\sum_{i<j}w_{ij}(\widetilde v_i^2+\widetilde v_j^2)\\
 &=2\sum_{i=1}^d K_{ii}\widetilde v_i^2\\
 &\le2h_q\sum_{i=1}^d\mu_i(v_i-\bar v)^2\\
 &=2h_q\left(\sum_{i=1}^d\mu_i v_i^2
                   -\inner{\mu,v}^2\right)
 =2h_q v^\top M(\Theta)v.
\end{align*}
Since this holds for every $v$, it proves $K\preceq2h_qM(\Theta)$.
\end{proof}
 
\begin{lemma}[Exposed-edge gadget]\label{lem:exposed-edge-gadget}
Let $\calA\subset\R^2$ be compact and convex, and suppose that
$E=[a_-,a_+]$ is an exposed edge: for some unit vector $n$, $\argmax_{a\in\calA}\inner{a,n}=E$. 
Write $ \ell=\norm{a_+-a_-}_2$ and $t=\frac{a_+-a_-}{\ell}$. 
Let $\rho$ be a regular perturbation density on $\R^2$. Let
$x(\Theta)\in\argmax_{a\in\calA}\inner{a,\Theta}$ and define
\begin{align*}
 \Psi(\Theta)
 &=\E_{Z\sim\rho}\left[
   \max_{a\in\calA}\inner{a,\Theta+Z}\right],\\
 M(\Theta)
 &=\Cov_{Z\sim\rho}\bigl(x(\Theta+Z)\bigr).
\end{align*}
If
\begin{align*}
 \nabla^2\Psi(\Theta)\preceq cM(\Theta)
 \qquad\text{for every }\Theta\in\R^2,
\end{align*}
then, for every $R>0$ satisfying
$\Pr(\norm{Z}_2\le R)\ge\frac12$, one must have $c\ge\frac{1}{R\ell}$. 
\end{lemma}
 
\begin{proof}
For $\alpha>0$ and $\beta\in\R$, define
\begin{align*}
 m_\alpha(\beta)
 &=\inner{\nabla\Psi(\alpha n+\beta t),t},\\
 V_\alpha(\beta)
 &=t^\top M(\alpha n+\beta t)t.
\end{align*}
Differentiability of $\Psi$ and the assumed matrix inequality give
\begin{align}
 m_\alpha(R)-m_\alpha(-R)
 &=\int_{-R}^R
   t^\top\nabla^2\Psi(\alpha n+\beta t)t\,\d\beta\nonumber\\
 &\le c\int_{-R}^R V_\alpha(\beta)\,\d\beta.
 \label{eq:edge-gadget-integrated}
\end{align}
 
Let $Y=\inner{t,Z}$. For fixed $\beta$ and almost every $Z$, compactness and
the exposed-face property imply
\begin{align*}
 x(\alpha n+\beta t+Z)
 \longrightarrow
 \begin{cases}
  a_+, & \beta+Y>0,\\
  a_-, & \beta+Y<0,
 \end{cases}
 \qquad\text{as }\alpha\to\infty.
\end{align*}
The scalar random variable $Y$ has a density, so the tie event has
probability zero. Writing $p_\beta=\Pr(Y>-\beta)$, bounded convergence yields
\begin{align*}
 m_\alpha(\beta)&\longrightarrow
 p_\beta\inner{a_+,t}+(1-p_\beta)\inner{a_-,t},\\
 V_\alpha(\beta)&\longrightarrow \ell^2p_\beta(1-p_\beta).
\end{align*}
The same boundedness permits dominated convergence in the integral in
\cref{eq:edge-gadget-integrated}. Therefore,
\begin{align*}
 \lim_{\alpha\to\infty}
 \bigl(m_\alpha(R)-m_\alpha(-R)\bigr)
 &=\ell\Pr(-R<Y\le R)
 \ge\frac{\ell}{2},\\
 \lim_{\alpha\to\infty}
 \int_{-R}^R V_\alpha(\beta)\,\d\beta
 &=\ell^2\int_{-R}^R p_\beta(1-p_\beta)\,\d\beta
 \le\frac{R\ell^2}{2}.
\end{align*}
Taking $\alpha\to\infty$ in \cref{eq:edge-gadget-integrated} gives
$\ell/2\le cR\ell^2/2$, which proves the claim.
\end{proof}
 
\begin{proof}[Proof of \cref{thm: impossible}]
The generating set in the definition of $\calA$ is compact because
$v_i\to e_1$, so $\calA$ is compact. Every generating point lies in
$B_2^2$, and hence $\calA\subseteq B_2^2$. Moreover,
\begin{align*}
 \operatorname{conv}\{\pm e_1,\pm e_2\}
 =\{a\in\R^2:\norm{a}_1\le1\}
 \supseteq\frac{1}{\sqrt2}B_2^2
 \supseteq\frac12B_2^2.
\end{align*}
 
For every $i\ge2$, the segment $E_i=[v_i,v_{i+1}]$ is an exposed edge
of $\calA$. Indeed, set $\delta_i=2^{-i}\pi$ and write
\begin{align*}
 \bar\delta_i=\frac{\delta_i+\delta_{i+1}}{2},
 \qquad
 n_i=(\cos\bar\delta_i,\sin\bar\delta_i),
\end{align*}
the two endpoints $v_i$ and $v_{i+1}$ are the only generating points
maximizing $\inner{a,n_i}$. The edge length is
\begin{align*}
 \ell_i
 =\norm{v_i-v_{i+1}}_2
 =2\sin\left(\frac{\delta_i-\delta_{i+1}}{2}\right)
 =2\sin\left(\frac{\pi}{2^{i+2}}\right)
 \longrightarrow0.
\end{align*}
 
Fix any regular perturbation density $\rho$, and choose $R<\infty$ such
that $\Pr(\norm{Z}_2\le R)\ge\frac12$. If a finite constant $c$ satisfied
$\nabla^2\Psi(\Theta)\preceq cM(\Theta)$ for every $\Theta$, then
\cref{lem:exposed-edge-gadget} applied to every $E_i$ would imply
\begin{align*}
 c\ge\frac{1}{R\ell_i}
 \qquad\text{for every }i\ge2.
\end{align*}
Since $\ell_i\to0$, no finite $c$ can satisfy these inequalities.
\end{proof}
 
\section{Proofs of \cref{thm:A-bound}, \cref{lem:convolved-shifted}, \cref{thm:B-fisher}, \cref{thm:B-elliptical}}\label[appendix]{app:general-B}
This appendix proves
\cref{thm:A-bound,lem:convolved-shifted,thm:B-fisher,thm:B-elliptical},
using $\calI_t^v$ for random-weighting perturbations and
$\calJ^v(P_\Sigma)$ for elliptical perturbations.
 
\begin{lemma}[Hessian of a smoothed support function]\label{lem:smoothed-support-hessian}
Call a perturbation density $\rho$ \emph{regular} if it is locally Lipschitz, $\E_{Z\sim\rho}[\norm Z_2]<\infty$, and $\int\norm{\nabla\rho(z)}_2\,\d z<\infty$; all perturbations in this paper are regular. For a regular $\rho$, define
$\Psi(\Theta)=\E_{Z\sim\rho}[h(\Theta+Z)]$.
Then
\begin{align*}
 \nabla\Psi(\Theta)=\E_{Z\sim\rho}[x(\Theta+Z)],
 \qquad
 \nabla^2\Psi(\Theta)
 =-\int_{\R^d}x(z)\nabla\rho(z-\Theta)^\top\,\d z.
\end{align*}
\end{lemma}

\begin{proof}
Since $h(\Theta+z)\le R\norm{\Theta+z}_2$ and $\E\norm Z_2<\infty$, $\Psi$ is finite.

\emph{Gradient.} The support function $h$ is convex and $R$-Lipschitz, and it is differentiable with $\nabla h=x$ almost everywhere. Since $Z$ has a density, for every direction $u$ the difference quotient $\frac1\varepsilon\bigl(h(\Theta+\varepsilon u+Z)-h(\Theta+Z)\bigr)$ is bounded by $R\norm u_2$ and converges to $\inner{x(\Theta+Z),u}$ almost surely. Taking expectations (dominated convergence) and substituting $z\leftarrow\Theta+z$ gives
\begin{align*}
 \nabla\Psi(\Theta)=\int_{\R^d}x(\Theta+z)\rho(z)\,\d z=\int_{\R^d}x(z)\rho(z-\Theta)\,\d z.
\end{align*}

\emph{Hessian.} In the last expression only the density depends on $\Theta$, so we can differentiate without differentiating $x$. Since $\rho$ is locally Lipschitz, $\rho(y-\varepsilon u)-\rho(y)=-\varepsilon\int_0^1\inner{\nabla\rho(y-s\varepsilon u),u}\,\d s$, and hence
\begin{align*}
 \frac{\nabla\Psi(\Theta+\varepsilon u)-\nabla\Psi(\Theta)}{\varepsilon}
 =-\int_0^1\!\!\int_{\R^d}x(z)\inner{\nabla\rho(z-\Theta-s\varepsilon u),u}\,\d z\,\d s.
\end{align*}
For any integrable $f$, $\int\norm{f(y-w)-f(y)}_2\,\d y\to0$ as $w\to0$. Applying this to $f=\nabla\rho$ and using $\norm{x(z)}_2\le R$, the right-hand side converges to $-\int x(z)\nabla\rho(z-\Theta)^\top u\,\d z$, which is the claimed Hessian applied to $u$. 
\end{proof}
 
\subsection{Proof of \cref{thm:A-bound}}\label[appendix]{app:A-bound}
\begin{proof}[Proof of \cref{thm:A-bound}]
Condition on the history through round $t$.
\cref{lem:smoothed-support-hessian} gives
$\nabla\Psi_t(\Thetahat_t)=\mu_t$. For $s\in[0,1]$, its Hessian
identity, a change of variables, and
$\int\nabla\rho_t(z-s\thetahat_t)\,\d z=0$ give
\begin{align*}
 &\thetahat_t^\top\nabla^2\Psi_t(\Thetahat_t+s\thetahat_t)\thetahat_t\\
 &=-\int_{\R^d}  \inner{x(z), \thetahat_t} \inner{\nabla \rho_t(z-\Thetahat_t - s\thetahat_t), \thetahat_t}\,\d z  \\
 &=-\int_{\R^d} \inner{x(\Thetahat_t + z), \thetahat_t} \inner{\nabla \rho_t(z -  s\thetahat_t),\thetahat_t}\,\d z \\   
 &=-\int_{\R^d}\inner{x(\Thetahat_t+z)-\mu_t,\thetahat_t}
                   \inner{\nabla\rho_t(z-s\thetahat_t),\thetahat_t}\,\d z\\
 &\le\left(\int_{\R^d}\inner{x(\Thetahat_t+z)-\mu_t,\thetahat_t}^2
                         \rho_t(z)\,\d z\right)^{1/2}
          \left(\int_{\R^d}
          \frac{\inner{\nabla\rho_t(z-s\thetahat_t),\thetahat_t}^2}
               {\rho_t(z)}\,\d z\right)^{1/2}\\
 &=\sqrt{\thetahat_t^\top M_t\thetahat_t}
         \sqrt{\thetahat_t^\top\calI_t^{s\thetahat_t}\thetahat_t}.
\end{align*}
The inequality is weighted Cauchy--Schwarz, and the last equality uses
\cref{def:fisher-information} after translating the second integral.
The bound is vacuous if that integral is infinite.
If $\thetahat_t=0$, then $A_t=0$. Otherwise, by the Taylor remainder theorem,
\begin{align*}
 A_t
 &\le\frac12\sup_{s\in[0,1]}\thetahat_t^\top
       \nabla^2\Psi_t(\Thetahat_t+s\thetahat_t)\thetahat_t\\
 &\le\frac12\sqrt{\thetahat_t^\top M_t\thetahat_t}
       \sup_{s\in[0,1]}\sqrt{\thetahat_t^\top\calI_t^{s\thetahat_t}\thetahat_t},
\end{align*}
which is~\cref{eq:A-bound}.
\end{proof}
 
\subsection{Proof of \cref{lem:convolved-shifted}}\label[appendix]{app:convolution-properties}
Condition on the history through round $t$. Let $\rho_t$ and
$\rho_{t+1}$ be the densities under random-weighting perturbations in round $t$ and $t+1$ respectively. Then it holds that 
\begin{align*}
 \rho_{t+1}(z)=\E_{\xi\sim \varphi}[\rho_t(z-\xi\thetahat_t)].   \numberthis \label{eq: update convov}
\end{align*} 
\begin{lemma}
\label{lem:convolution-ratio}
Suppose $\rho_t$ is positive almost everywhere. For every shift $v$,
\begin{equation}\label{eq:convolution-ratio-contraction}
 \int_{\R^d}\frac{\rho_{t+1}(z-v)^2}{\rho_{t+1}(z)}\,\d z
 \le \int_{\R^d}\frac{\rho_t(z-v)^2}{\rho_t(z)}\,\d z.
\end{equation}
\end{lemma}
\begin{proof}
For any fixed $z$, we have
\begin{align*}
 \rho_{t+1}(z-v)^2
 &=\left(\E_\xi[\rho_t(z-v - \xi\thetahat_t)]\right)^2   \tag{\cref{eq: update convov}}\\
 &\le\underbrace{\E_\xi[\rho_t(z-\xi\thetahat_t)]}_{=\rho_{t+1}(z)}\,
       \E_\xi\left[\frac{\rho_t(z-v-\xi\thetahat_t)^2}
                         {\rho_t(z-\xi\thetahat_t)}\right].  \tag{Cauchy-Schwarz}
 \end{align*}
Divide by $\rho_{t+1}(z)$ and integrate:
\begin{align*}
 \int\frac{\rho_{t+1}(z-v)^2}{\rho_{t+1}(z)}\,\d z
 &\le\E_\xi\int
   \frac{\rho_t(z-v-\xi\thetahat_t)^2}
        {\rho_t(z-\xi\thetahat_t)}\,\d z\\
 &=\E_\xi\int\frac{\rho_t(x-v)^2}{\rho_t(x)}\,\d x
 \tag{substitute $x = z- \xi\thetahat_t$ } \\
 &=\int\frac{\rho_t(z-v)^2}{\rho_t(z)}\,\d z.
\end{align*}
\end{proof}
 
\begin{lemma}
\label{lem:convolution-derivative}
Let $\rho_t$ satisfy \cref{eq: update convov}.
Suppose $\varphi$ is a probability density on $\R$
with $I_{\varphi}=\int_{\xi} \frac{(\varphi'(\xi))^2}{\varphi(\xi)}\,\d \xi <\infty$.
For any test direction $\theta\in\R^d$, shift $v$, and scalar $\alpha$, the following
holds:
\begin{equation}\label{eq:convolution-derivative-splitting}
 \begin{aligned}
 &\int_{\R^d}\frac{(\inner{\theta,\nabla\rho_{t+1}(z-v)})^2}
                        {\rho_{t+1}(z)}\,\d z
\le\int_{\R^d}
   \frac{(\inner{\theta-\alpha\thetahat_t,\nabla\rho_t(z-v)})^2}
        {\rho_t(z)}\,\d z +\alpha^2 I_{\varphi}
    \int_{\R^d}\frac{\rho_t(z-v)^2}{\rho_t(z)}\,\d z.
 \end{aligned}
\end{equation}
\end{lemma}
\begin{proof}
The directional derivative satisfies
\begin{align*}
 \left.\frac{\d}{\d \varepsilon}\rho_{t+1}(z-v+\varepsilon \theta)
 \right|_{\varepsilon=0}
 =\inner{\theta,\nabla\rho_{t+1}(z-v)}.
\end{align*}
To calculate it, first rewrite the translated density:
\begin{align*}
 \rho_{t+1}(z-v+\varepsilon \theta)
 &=\int_\R\rho_t(z-v+\varepsilon \theta-\xi\thetahat_t)\varphi(\xi)\,\d \xi \tag{\cref{eq: update convov}}\\
 &=\int_\R\rho_t\!\left(z-v+\varepsilon \theta
       -(\xi'+\varepsilon\alpha)\thetahat_t\right)\cdot \varphi(\xi'+\varepsilon\alpha)\,\d \xi' \\
 &=\int_\R\rho_t\left(z-v+\varepsilon(\theta-\alpha\thetahat_t)-\xi\thetahat_t\right)\cdot \varphi(\xi+\varepsilon\alpha)\,\d \xi
\end{align*}
Take derivative over $\varepsilon$ at $\varepsilon=0$ on both sides:
\begin{align*}
 \inner{\theta,\nabla\rho_{t+1}(z-v)}
 &=\int_\R\Bigl[
     \inner{\theta-\alpha\thetahat_t,\nabla\rho_t(z-v-\xi\thetahat_t)}\varphi(\xi)+\rho_t(z-v-\xi\thetahat_t)\alpha \varphi'(\xi)\Bigr]\,\d \xi
 \\
 &=\int_\R\Biggl[
     \inner{\theta-\alpha\thetahat_t,\nabla\rho_t(z-v-\xi\thetahat_t)}+\rho_t(z-v-\xi\thetahat_t)
     \alpha\frac{\varphi'(\xi)}{\varphi(\xi)}\Biggr]\varphi(\xi)\,\d \xi\\
 &=\E_\xi\Biggl[
     \inner{\theta-\alpha\thetahat_t,\nabla\rho_t(z-v-\xi\thetahat_t)}+\rho_t(z-v-\xi\thetahat_t)
     \alpha\frac{\varphi'(\xi)}{\varphi(\xi)}\Biggr].
\end{align*}
By Cauchy--Schwarz,
\begin{align*}
 \bigl(\inner{\theta,\nabla\rho_{t+1}(z-v)}\bigr)^2
 &\le\underbrace{\E_\xi[\rho_t(z-\xi\thetahat_t)]}_{=\rho_{t+1}(z)}\,
   \E_\xi\left[\frac{\left(\inner{\theta-\alpha\thetahat_t,\nabla\rho_t(z-v-\xi\thetahat_t)}+\rho_t(z-v-\xi\thetahat_t)
     \alpha\frac{\varphi'(\xi)}{\varphi(\xi)}\right)^2}{\rho_t(z-\xi\thetahat_t)}\right].
\end{align*}
Therefore,
\begin{align*}
 \int\frac{(\inner{\theta,\nabla\rho_{t+1}(z-v)})^2}{\rho_{t+1}(z)}\,\d z
 &\le\E_\xi\int
   \frac{\left(\inner{\theta-\alpha\thetahat_t,\nabla\rho_t(z-v-\xi\thetahat_t)}+\rho_t(z-v-\xi\thetahat_t)
     \alpha\frac{\varphi'(\xi)}{\varphi(\xi)}\right)^2}{\rho_t(z-\xi\thetahat_t)}\,\d z\\
 &=\E_\xi\int
   \frac{\left(\inner{\theta-\alpha\thetahat_t,\nabla\rho_t(x-v)}
     +\rho_t(x-v)\alpha\,\frac{\varphi'(\xi)}{\varphi(\xi)}\right)^2}
        {\rho_t(x)}\,\d x.
\end{align*}
After this substitution, only $\varphi'(\xi)/\varphi(\xi)$ depends on $\xi$.
Thus the cross term in the square is
\begin{align*}
 2\alpha
 \left(\int\frac{\inner{\theta-\alpha\thetahat_t,\nabla\rho_t(x-v)}\rho_t(x-v)}
                    {\rho_t(x)}\,\d x\right)
 \underbrace{\E_\xi\left[\frac{\varphi'(\xi)}{\varphi(\xi)}\right]}_{=0}=0.
\end{align*}
The remaining terms give the right-hand side of~\cref{eq:convolution-derivative-splitting}, by noticing that
$I_{\varphi}=\E_\xi[(\varphi'(\xi)/\varphi(\xi))^2]$.
\end{proof}

\begin{proof}[Proof of \cref{lem:convolved-shifted}]
Fix $v$ and let $C$ denote the maximum on the right-hand side
of~\cref{eq:convolved-general-derivative}. If $C=\infty$, the bound is
vacuous; otherwise, it is equivalent to proving, for every direction $\theta$,
\begin{align*}
 \int_{\R^d}\frac{(\inner{\theta,\nabla\rho_t(z-v)})^2}{\rho_t(z)}\,\d z
 \le C \theta^\top\Sigma_t^{-1}\theta.
\end{align*}
By repeated application of \cref{lem:convolution-ratio},
\begin{equation}\label{eq:convolved-ratio-induction}
 \int_{\R^d}\frac{\rho_t(z-v)^2}{\rho_t(z)}\,\d z
 \le\int_{\R^d}\frac{\rho_1(z-v)^2}{\rho_1(z)}\,\d z.
\end{equation}
We prove this directional bound by induction, simultaneously for every
$\theta\in\R^d$. At $t=1$, the definition of $C$ gives
\begin{align*}
 \int_{\R^d}\frac{(\inner{\theta,\nabla\rho_1(z-v)})^2}{\rho_1(z)}\,\d z
 &\le \|\theta\|_2^2 \sup_{\|u\|_2\leq 1}\int_{\R^d}\frac{(\inner{u,\nabla\rho_1(z-v)})^2}{\rho_1(z)}\,\d z \\
 &\leq  \frac{C}{\lambda}\norm \theta_2^2
 =C \theta^\top\Sigma_1^{-1}\theta.
\end{align*}
For the induction step, assume the bound holds at round $t$ for every
direction. Fix $\theta$ and put $u=\Sigma_{t+1}^{-1}\theta$.
The matrix update implies
\begin{align*}
 \theta-\inner{\thetahat_t,u}\thetahat_t=\Sigma_tu.
\end{align*}
Apply \cref{lem:convolution-derivative}
with $\alpha=\inner{\thetahat_t,u}$:
\begin{align*}
 &\int_{\R^d}\frac{(\inner{\theta,\nabla\rho_{t+1}(z-v)})^2}
                        {\rho_{t+1}(z)}\,\d z\\
 &\quad\le\int_{\R^d}
   \frac{(\inner{\Sigma_tu,\nabla\rho_t(z-v)})^2}{\rho_t(z)}\,\d z
   +I_{\varphi}\inner{\thetahat_t,u}^2
      \int_{\R^d}\frac{\rho_t(z-v)^2}{\rho_t(z)}\,\d z\\
 &\quad\le C u^\top\Sigma_tu+C\inner{\thetahat_t,u}^2   \tag{by the induction hypothesis, \cref{eq:convolved-ratio-induction}, and the definition of $C$}\\
 &\quad=C u^\top\Sigma_{t+1}u
 =C \theta^\top\Sigma_{t+1}^{-1}\theta.
\end{align*}
This completes the induction.
\end{proof}

\subsection{Proof of \cref{thm:B-fisher}}\label[appendix]{app:B-fisher}

\begin{proof}[Proof of \cref{thm:B-fisher}]
The convolution update and
$\E[\xi]=0$ imply
\begin{align*}
 B_t=\E_{\xi\sim \varphi}\bigl[\Psi_t(\Thetahat_{t+1}+\xi\thetahat_t)
 -\Psi_t(\Thetahat_{t+1})
 -\xi\inner{\nabla\Psi_t(\Thetahat_{t+1}),\thetahat_t}\bigr].
\end{align*}
For any shift $v$, \cref{lem:smoothed-support-hessian} and
$\int\nabla\rho_t(z-v)\,\d z=0$ give, by translating and centering,
\begin{align*}
 &\thetahat_t^\top\nabla^2\Psi_t(\Thetahat_t+v)\thetahat_t\\
 &\quad=-\int_{\R^d}\inner{x(\Thetahat_t+z)-\mu_t,\thetahat_t}
                   \inner{\nabla\rho_t(z-v),\thetahat_t}\,\d z\\
 &\quad\le\left(\int_{\R^d}\inner{x(\Thetahat_t+z)-\mu_t,\thetahat_t}^2
                         \rho_t(z)\,\d z\right)^{1/2}
          \left(\int_{\R^d}
          \frac{\inner{\nabla\rho_t(z-v),\thetahat_t}^2}
               {\rho_t(z)}\,\d z\right)^{1/2}\\
 &\quad=\sqrt{\thetahat_t^\top M_t\thetahat_t}
         \sqrt{\thetahat_t^\top\calI_t^v\thetahat_t}.
\end{align*}
Here we used Cauchy--Schwarz and
\cref{def:fisher-information}.
For each fixed $\xi$, by the Taylor remainder theorem followed by the
preceding estimate with $v=(1+s\xi)\thetahat_t$, we obtain
\begin{align*}
 &\Psi_t(\Thetahat_{t+1}+\xi\thetahat_t)-\Psi_t(\Thetahat_{t+1})
 -\xi\inner{\nabla\Psi_t(\Thetahat_{t+1}),\thetahat_t}\\
 &\quad\le\frac{\xi^2}{2}\sup_{s\in[0,1]}
 \thetahat_t^\top\nabla^2\Psi_t(\Thetahat_t+(1+s\xi)\thetahat_t)
 \thetahat_t\\
 &\quad\le\frac{\xi^2}{2}\sqrt{\thetahat_t^\top M_t\thetahat_t}
 \sup_{s\in[0,1]}
 \sqrt{\thetahat_t^\top\calI_t^{(1+s\xi)\thetahat_t}\thetahat_t}.
\end{align*}
\textit{Handling large scalar increments.} Since $\Psi_t$ is $R$-Lipschitz
and $\norm{\nabla\Psi_t(\Thetahat_{t+1})}_2\le R$, we have 
\begin{align*}
 0\le{}&\Psi_t(\Thetahat_{t+1}+\xi\thetahat_t)
       -\Psi_t(\Thetahat_{t+1})
       -\xi\inner{\nabla\Psi_t(\Thetahat_{t+1}),\thetahat_t}
 \le{}2R|\xi|\norm{\thetahat_t}_2.
\end{align*}
Thus, for a threshold $\sfL\ge1$, we have 
\begin{align*}
   B_t &= \E_{\xi\sim \varphi}\bigl[\Psi_t(\Thetahat_{t+1}+\xi\thetahat_t)
 -\Psi_t(\Thetahat_{t+1})
 -\xi\inner{\nabla\Psi_t(\Thetahat_{t+1}),\thetahat_t}\bigr] \\
 &\leq \E_{\xi\sim \varphi}\left[\mathbf{1}_{\{|\xi|\leq \sfL\}}\left(\Psi_t(\Thetahat_{t+1}+\xi\thetahat_t)
 -\Psi_t(\Thetahat_{t+1})
 -\xi\inner{\nabla\Psi_t(\Thetahat_{t+1}),\thetahat_t}\right)\right] + \E_{\xi\sim \varphi} \left[\mathbf{1}_{|\xi|>\sfL} |\xi|\right]\cdot 2R\|\thetahat_t\|_2 \\ 
 &\leq \E_{\xi\sim \varphi}\left[\frac{\xi^2}{2}\mathbf{1}_{\{|\xi|\leq \sfL\}}\sqrt{\thetahat_t^\top M_t\thetahat_t}
 \sup_{s\in[0,1]}
 \sqrt{\thetahat_t^\top\calI_t^{(1+s\xi)\thetahat_t}\thetahat_t}\right] + \E_{\xi\sim \varphi} \left[\mathbf{1}_{|\xi|>\sfL} |\xi|\right]\cdot 2R\|\thetahat_t\|_2. 
\end{align*}

\end{proof}
 
\subsection{Proof of \cref{thm:B-elliptical}}\label[appendix]{app:B-elliptical}
 
\begin{proof}[Proof of \cref{thm:B-elliptical}]
Condition on the history through round $t$ and write
$\Sigma=\Sigma_t$ and
$\Sigma_s=\Sigma+s\thetahat_t\thetahat_t^\top$ for $s\in[0,1]$.
We use a change of variables to express the changing perturbation law
in terms of the fixed density $P_\Sigma$. To this end, choose
\begin{align*}
 L_s=I+\frac{s}{\sqrt{1+s\norm{\thetahat_t}_{\Sigma^{-1}}^2}+1}
            \thetahat_t\thetahat_t^\top\Sigma^{-1}.
\end{align*}
This map stretches the perturbations to the new elliptical shape:
\begin{align*}
 L_0=I,\qquad L_s\Sigma L_s^\top=\Sigma_s,\qquad
 \det L_s=\sqrt{1+s\norm{\thetahat_t}_{\Sigma^{-1}}^2}>0.   \numberthis \label{eq: to verify}
\end{align*}
To verify the middle equality in \cref{eq: to verify}, denote $\alpha_s = \frac{s}{\sqrt{1+s\|\thetahat_t\|^2_{\Sigma^{-1}}}+1}$. Then 
\begin{align*}
   L_s\Sigma L_s^\top &= (I+\alpha_s \thetahat_t\thetahat_t^\top \Sigma^{-1})\Sigma (I+\alpha_s \Sigma^{-1} \thetahat_t\thetahat_t^\top ) \\ 
   &= \Sigma + 2\alpha_s \thetahat_t \thetahat_t ^\top + \alpha_s^2 \thetahat_t \thetahat_t ^\top \Sigma^{-1}
\thetahat_t \thetahat_t ^\top  \\ 
&= \Sigma + \left(2\alpha_s + \alpha_s^2 \|\thetahat_t\|^2_{\Sigma^{-1}}\right) \thetahat_t \thetahat_t ^\top \\
&= \Sigma + s \thetahat_t \thetahat_t ^\top \tag{by the definition of $\alpha_s$} \\
&= \Sigma_s. 
\end{align*}
The last equality in \cref{eq: to verify} can be verified by the formula $\det(I+ab^\top) = 1+b^\top a$: 
\begin{align*}
    \det L_s = 1+ \frac{s\|\thetahat_t\|_{\Sigma^{-1}}^2}{\sqrt{1 + s\|\thetahat_t\|_{\Sigma^{-1}}^2}+1} = \sqrt{1 + s\|\thetahat_t\|^2_{\Sigma^{-1}}}.  
\end{align*}
Thus, by \cref{lem:elliptical-linear-transformation}, we have that $L_s w\sim P_{\Sigma_s} = P_{L_s\Sigma L_s^\top}$ is equivalent to $w\sim P_\Sigma$.  With substitution $z=L_s w$, define 
\begin{align*}
 F(s)&=\int h(\Thetahat_{t+1}+z)P_{\Sigma_s}(z)\,\d z
      =\int h(\Thetahat_{t+1}+L_sw)P_\Sigma(w)\,\d w.
\end{align*}
Note that $B_t = F(1) - F(0)$. 

We have 
\begin{align*}
 \frac{\d L_s}{\d s}w
 =\frac{\thetahat_t(\thetahat_t^\top\Sigma^{-1}w)}
        {2\sqrt{1+s\norm{\thetahat_t}_{\Sigma^{-1}}^2}}
 =\frac12\thetahat_t(\thetahat_t^\top\Sigma_s^{-1}L_sw)
\end{align*}
and $\nabla h=x$. Since \(h\) is \(R\)-Lipschitz, \(L_s\) is continuously differentiable, and \(P_\Sigma\) has finite first moment, dominated convergence justifies differentiation under the integral. Thus, we obtain
\begin{align*}
 F'(s)
 &=\int\Inner{x(\Thetahat_{t+1}+L_sw),\frac{\d L_s}{\d s}w}
 P_\Sigma(w)\,\d w\\
 &=\frac12\int
 \inner{x(\Thetahat_{t+1}+z),\thetahat_t}
 (\thetahat_t^\top\Sigma_s^{-1}z)P_{\Sigma_s}(z)\,\d z.  \tag{changing back to $z=L_sw$}
\end{align*}
By symmetry, $P_{\Sigma_s}$ has mean zero. Integrating $F'(s)$ over $s$ and subtracting
$\inner{\mu_t,\thetahat_t}$ therefore gives, with
$\kappa=\norm{\thetahat_t}_{\Sigma^{-1}}$,
$b_s=1+s\kappa^2$, and $b=1+\kappa^2$,
\begin{align*}
 B_t
 &=\frac12\int_{\R^d}
 \inner{x(\Thetahat_{t+1}+z)-\mu_t,\thetahat_t}
 (\thetahat_t^\top\Sigma^{-1}z)
 \left(\int_0^1\frac{P_{\Sigma_s}(z)}{b_s}\,\d s\right)\d z.
\end{align*}
Here we use the Sherman--Morrison identity
$\thetahat_t^\top\Sigma_s^{-1}z
=(\thetahat_t^\top\Sigma^{-1}z)/b_s$.
The matrix determinant lemma gives
$\det(\Sigma_s)=\det(\Sigma)b_s$, while
$\Sigma_s\preceq\Sigma_1$ implies
$\norm z_{\Sigma_s^{-1}}\ge\norm z_{\Sigma_1^{-1}}$.
Since the radial profile $f$ is nonincreasing, we have by \cref{lem:elliptical-linear-transformation}, 
\begin{align*}
   P_{\Sigma_s}(z) = \frac{f(\|z\|_{\Sigma_s^{-1}})}{C_f \sqrt{\det \Sigma_s}} \leq \frac{f(\|z\|_{\Sigma_1^{-1}})}{C_f \sqrt{\det \Sigma_s}} = \sqrt{\frac{\det \Sigma_1}{\det \Sigma_s}}\frac{f(\|z\|_{\Sigma_1^{-1}})}{C_f \sqrt{\det \Sigma_1}} = \sqrt{\frac{b}{b_s}} P_{\Sigma_1}(z). 
\end{align*}
Consequently,
\begin{align*}
 \int_0^1\frac{P_{\Sigma_s}(z)}{b_s}\,\d s
 &\le P_{\Sigma_1}(z)\sqrt b
 \int_0^1(1+s\kappa^2)^{-3/2}\,\d s\\
 &=\frac{2}{1+\sqrt b}P_{\Sigma_1}(z).
\end{align*}
It follows that
\begin{align*}
 B_t
 &\le\frac{1}{1+\sqrt b}\int_{\R^d}
 \left|\inner{x(\Thetahat_{t+1}+z)-\mu_t,\thetahat_t}\right|
 \left|\thetahat_t^\top\Sigma^{-1}z\right|P_{\Sigma_1}(z)\,\d z.    \numberthis \label{eq: split above}
\end{align*}
Applying Cauchy--Schwarz after multiplying and dividing by
$P_\Sigma(z+\thetahat_t)^{1/2}$: 
\begin{align*}
    B_t &\leq \frac{1}{1+\sqrt{b}} \sqrt{\int_{\R^d} \inner{x(\Thetahat_{t+1}+z)-\mu_t,\thetahat_t}^2 P_\Sigma(z+\thetahat_t)\d z} \sqrt{\int_{\R^d} \frac{P_{\Sigma_1}(z)^2}{P_\Sigma(z+\thetahat_t)} \thetahat_t^\top \Sigma^{-1}zz^\top \Sigma^{-1}\thetahat_t \d z}  \\
    &=\frac{1}{1+\sqrt b}
 \sqrt{\thetahat_t^\top M_t\thetahat_t}
 \sqrt{\thetahat_t^\top
 \calJ^{\thetahat_t}(P_\Sigma)\thetahat_t}  \tag{substitute $w=z+\thetahat_t$}\\
 &\le\frac12
 \sqrt{\thetahat_t^\top M_t\thetahat_t}
 \sqrt{\thetahat_t^\top
 \calJ^{\thetahat_t}(P_\Sigma)\thetahat_t},
\end{align*}
which proves \cref{eq:B-elliptical-unclipped}.
 
For the clipped bound, fix a finite clipping radius
$\sfL\ge\sqrt b$ and split the nonnegative
integral \cref{eq: split above} over $\calE_\Sigma(\sfL)$ and its complement.
Inside the ellipsoid, apply Cauchy--Schwarz after multiplying and dividing
by $P_\Sigma(z+\thetahat_t)^{1/2}$. The squared action factor is bounded by
\begin{align*}
 &\int_{\calE_\Sigma(\sfL)}
 \inner{x(\Thetahat_{t+1}+z)-\mu_t,\thetahat_t}^2
 P_\Sigma(z+\thetahat_t)\,\d z\\
 &\qquad\le\int_{\R^d}
 \inner{x(\Thetahat_{t+1}+z)-\mu_t,\thetahat_t}^2
 P_\Sigma(z+\thetahat_t)\,\d z
 =\thetahat_t^\top M_t\thetahat_t,
\end{align*}
where we used $w=z+\thetahat_t$ and
$\Thetahat_{t+1}=\Thetahat_t+\thetahat_t$. The squared distribution factor is
\begin{align*}
 \int_{\calE_\Sigma(\sfL)}
 \frac{P_{\Sigma_1}(z)^2(\thetahat_t^\top\Sigma^{-1}z)^2}
 {P_\Sigma(z+\thetahat_t)}\,\d z
 =\thetahat_t^\top\calJ^{\thetahat_t,\sfL}(P_\Sigma)\thetahat_t.
\end{align*}
Using $1+\sqrt b\ge2$, the interior contribution is therefore at most
\begin{align*}
 \frac12\sqrt{\thetahat_t^\top M_t\thetahat_t}
 \sqrt{\thetahat_t^\top\calJ^{\thetahat_t,\sfL}(P_\Sigma)\thetahat_t}.
\end{align*}
 
For the complement, every action and $\mu_t$ have Euclidean norm at most
$R$, so
\begin{align*}
 \left|\inner{x(\Thetahat_{t+1}+z)-\mu_t,\thetahat_t}\right|
 \le2R\norm{\thetahat_t}_2.
\end{align*}
Let $Z_+\sim P_{\Sigma_1}$. By
\cref{lem:elliptical-radial-representation}, write
$Z_+=\Sigma_1^{1/2}SU$, where $U$ is uniform on the unit sphere and
independent of $S$, and $S$ has the same radial law as under $P_\Sigma$.
Since $\Sigma_1\preceq b\Sigma$, we have
$\norm{Z_+}_{\Sigma^{-1}}\le\sqrt b S$.
Also, $\E[UU^\top]=I/d$ and Cauchy--Schwarz give
\begin{align*}
 \E_U\left[\left|\inner{\Sigma_1^{1/2}\Sigma^{-1}\thetahat_t,U}\right|\right]
 &\le\sqrt{\frac{\thetahat_t^\top\Sigma^{-1}\Sigma_1\Sigma^{-1}\thetahat_t}{d}} =\kappa\sqrt{\frac b d}.
\end{align*}
Consequently, independence of $S$ and $U$ gives
\begin{align*}
 &\E\left[\left|\thetahat_t^\top\Sigma^{-1}Z_+\right|
 \mathbf1_{\{Z_+\notin\calE_\Sigma(\sfL)\}}\right]\\
 &\qquad\le\E\left[S\mathbf1_{\{S>\sfL/\sqrt b\}}\right]
 \E_U\left[\left|\inner{\Sigma_1^{1/2}\Sigma^{-1}\thetahat_t,U}\right|\right]\\
 &\qquad\le\kappa\sqrt{\frac b d}
 \E\left[S\mathbf1_{\{S>\sfL/\sqrt b\}}\right].
\end{align*}
The exterior contribution is therefore at most
\begin{align*}
 \frac{2R\norm{\thetahat_t}_2\kappa\sqrt b}{(1+\sqrt b)\sqrt d}
 \E\left[S\mathbf1_{\{S>\sfL/\sqrt b\}}\right]
 \le\frac{2R\norm{\thetahat_t}_2\kappa}{\sqrt d}
 \E\left[S\mathbf1_{\{S>\sfL/\sqrt b\}}\right].
\end{align*}
Combining the interior and exterior contributions proves
\cref{eq:B-elliptical}. 
\end{proof}

\section{Proofs of \cref{cor:A-multiplier}, \cref{cor:B-multiplier}, \cref{cor:A-elliptical}, \cref{cor:B-elliptical}}
\label[appendix]{app:stability-corollaries}
We apply the bounds in \cref{app:perturbation-tables} with
$\kappa\le1/8$. Throughout, let $G$ be a deterministic upper bound on
$\norm{\thetahat_t}_2$, and write
$q_t=\sqrt{\thetahat_t^\top M_t\thetahat_t}$ and
$\kappa_t=\norm{\thetahat_t}_{\Sigma_t^{-1}}$.
 
\begin{proof}[Proof of \cref{cor:A-multiplier}]
All three initializers have shape matrix $\lambda I$.
Thus \cref{tab:convolution-initializers} gives, whenever
$\norm v_2\le\sqrt\lambda/8$,
\begin{align*}
 \lambda\norm{\calI_1^v}_{\op}=\calO(1),\qquad
 1+\chi^2(\rho_1(\cdot-v)\|\rho_1)=\calO(1).
\end{align*}
Together with $I_\varphi=\calO(1)$ from \cref{tab:scalar-noise},
\cref{lem:convolved-shifted} gives
$\calI_t^v\preceq \calO(1)\Sigma_t^{-1}$.
The assumed scale condition applies to every $v=s\thetahat_t$,
$s\in[0,1]$, so \cref{thm:A-bound} yields
$A_t\lesssim q_t\kappa_t$.
\end{proof}
 
\begin{proof}[Proof of \cref{cor:B-multiplier}]
Set $\sfL=C_0\log(RGT)$ with a sufficiently large universal $C_0$.
The tail bound in \cref{tab:scalar-noise} gives
$2RG\E[|\xi|\mathbf1_{\{|\xi|>\sfL\}}]\le T^{-1}$.
On $|\xi|\le\sfL$,
$\norm{(1+s\xi)\thetahat_t}_2\le(1+\sfL)G$ for all $s\in[0,1]$.
The assumed lower bound on $\lambda$ ensures that this is at most
$\sqrt\lambda/8$. The same argument as in the proof of \cref{cor:A-multiplier} therefore bounds $\calI^{(1+s\xi )\thetahat_t}_t\preceq \calO(1)\Sigma_t^{-1}$.  
Since $\E[\xi^2]=\calO(1)$, substitution into \cref{thm:B-fisher} gives
$B_t\lesssim q_t\kappa_t+T^{-1}$.
\end{proof}
 
\begin{proof}[Proof of \cref{cor:A-elliptical}]
Since $\Sigma_t\succeq\lambda I$ and
$\norm{\thetahat_t}_2\le\sqrt\lambda/8$, we have $\kappa_t\le1/8$.
The Fisher column of \cref{tab:elliptical-updates} therefore gives
$\calI_t^{s\thetahat_t}\preceq \calO(1)\Sigma_t^{-1}$ for every
$s\in[0,1]$. Apply \cref{thm:A-bound} to obtain
$A_t\lesssim q_t\kappa_t$.
\end{proof}
 
\begin{proof}[Proof of \cref{cor:B-elliptical}]
Again $\kappa_t\le1/8$. For Gaussian and radial-Laplace perturbations,
the untruncated update column of \cref{tab:elliptical-updates} and
\cref{eq:B-elliptical-unclipped} give $B_t\lesssim q_t\kappa_t$.
 
For Student-$t_2$, take
$\sfL=C_2RGT\sqrt d$ with a sufficiently large universal
$C_2$. The clipped update column of \cref{tab:elliptical-updates} gives
$\thetahat_t^\top\calJ_t^{\thetahat_t,\sfL}\thetahat_t
\lesssim\log(e\sfL)\kappa_t^2$.
Let $S=\norm Z_{\Sigma_t^{-1}}$ for $Z\sim P_{\Sigma_t}$.
Writing $b=1+\kappa_t^2\le65/64$, the radial-tail column gives
\begin{align*}
 \frac{2RG\kappa_t}{\sqrt d}
 \E\left[S\mathbf1_{\{S>\sfL/\sqrt b\}}\right]
 \lesssim\frac{RG\kappa_t\sqrt d}{\sfL}\le T^{-1}.
\end{align*}
Now \cref{eq:B-elliptical} and
$\log(e\sfL)\lesssim\log(dRGT)$ prove the claim.
\end{proof}
 
\section{Justifying \cref{tab:scalar-noise} Entries}\label[appendix]{app:scalar-table}\label[appendix]{app:distribution-tables}
Both scalar densities are symmetric and have mean zero. We calculate their Fisher information, second
moments, and clipped first-tail moments separately.
 
\paragraph{Gaussian.}
The normalized density is
$\varphi(x)=(2\pi)^{-1/2}e^{-x^2/2}$, with
$\varphi'(x)=-x\varphi(x)$. Integration by parts gives
\begin{align*}
 I_\varphi
 &=\int_{\R}\frac{\varphi'(x)^2}{\varphi(x)}\,\d x
 =\int_{\R}x^2\varphi(x)\,\d x = \E_{\xi\sim \varphi}[\xi^2]=1.
\end{align*}
For $\sfL\ge0$,
\begin{align*}
 \E[|\xi|\mathbf1_{\{|\xi|>\sfL\}}]
 &=2\int_{\sfL}^{\infty}x\varphi(x)\,\d x
 =-2\int_{\sfL}^{\infty}\varphi'(x)\,\d x\\
 &=2\varphi(\sfL)
 =\sqrt{2/\pi}\,e^{-\sfL^2/2}.
\end{align*}
 
\paragraph{Laplace.}
Here $\varphi(x)=e^{-|x|}/2$ and
$\varphi'(x)=-\operatorname{sign}(x)\varphi(x)$ almost everywhere.
Therefore,
\begin{align*}
 I_\varphi
 &=\int_{\R}\operatorname{sign}(x)^2\varphi(x)\,\d x=1,\\
 \E[\xi^2]
 &=\int_0^\infty x^2e^{-x}\,\d x
 =[-x^2e^{-x}]_0^\infty+2\int_0^\infty xe^{-x}\,\d x=2,\\
 \E[|\xi|\mathbf1_{\{|\xi|>\sfL\}}]
 &=\int_{\sfL}^\infty xe^{-x}\,\d x
 =[-(x+1)e^{-x}]_{\sfL}^\infty
 =(\sfL+1)e^{-\sfL}.
\end{align*}

\section{Justifying \cref{tab:convolution-initializers} and \cref{tab:elliptical-updates} Entries}\label[appendix]{app:elliptical-update-energy}
Throughout this section, $\Sigma\succ0$,
$\kappa=\norm v_{\Sigma^{-1}}\le1/8$, and $\sfL\ge1$.
The constants hidden by $\calO(1)$ are universal constants. 
 
\subsection{Common lemmas for radial distributions}\label[appendix]{app:elliptical-fisher}\label[appendix]{app:initializer-table}
\begin{lemma}[Linear transformation of an elliptical density]
\label{lem:elliptical-linear-transformation}
For a nonnegative integrable radial profile $f$, set
$C_f=\int_{\R^d}f(\norm y_2)\,\d y\in(0,\infty)$ and
\begin{align*}
 P_\Sigma(z)=\frac{f(\norm z_{\Sigma^{-1}})}{C_f\sqrt{\det\Sigma}}.
\end{align*}
For every invertible $A$ and positive definite $B$,
$P_{ABA^\top}(Ay)=|\det A|^{-1}P_B(y)$.
Consequently, $Y\sim P_B$ implies $AY\sim P_{ABA^\top}$.
\end{lemma}
\begin{proof}
The substitution $z=\Sigma^{1/2}y$ gives
\begin{align*}
 \int_{\R^d}f(\norm z_{\Sigma^{-1}})\,\d z
 =\sqrt{\det\Sigma}\int_{\R^d}f(\norm y_2)\,\d y
 =C_f\sqrt{\det\Sigma},
\end{align*}
which verifies the normalization. Also,
\begin{align*}
 \norm{Ay}_{(ABA^\top)^{-1}}^2
 &=y^\top A^\top A^{-\top}B^{-1}A^{-1}Ay=y^\top B^{-1}y,\\
 \sqrt{\det(ABA^\top)}&=|\det A|\sqrt{\det B}.
\end{align*}
Substitution in the normalized density yields
\begin{align*}
 P_{ABA^\top}(Ay)
 =\frac{f(\norm y_{B^{-1}})}{C_f|\det A|\sqrt{\det B}}
 =\frac{P_B(y)}{|\det A|}.
\end{align*}
Multiplying by the Jacobian $\d z=|\det A|\,\d y$ proves the
distributional assertion.
\end{proof}
 
\begin{lemma}[Radial representation of an elliptical density]
\label{lem:elliptical-radial-representation}
If $U$ is uniform on the Euclidean unit sphere and independent of
$S\ge0$ with density proportional to $s^{d-1}f(s)$, then
$Z=\Sigma^{1/2}SU$ has density $P_\Sigma$.
\end{lemma}
\begin{proof}
Consider first the density $P_I(z)\propto f(\|z\|_2)$. Since the surface area of a ball with radius $s$ is proportional to $s^{d-1}$, drawing $Z'\sim P_I$ is equivalent to drawing $S$ from a density proportional to $s^{d-1}f(s)$ and drawing $U$ uniformly from the unit sphere, and setting $Z' = SU$. \cref{lem:elliptical-linear-transformation} with $A=\Sigma^{1/2}$ implies that $Z=\Sigma^{1/2}SU$ follows $Z\sim P_\Sigma$. 
 
\end{proof}
 
\begin{lemma}[Radial moment and Fisher formulas]\label{lem:elliptical-moments}
Consider $Z\sim P_\Sigma$ with $S=\|Z\|_{\Sigma^{-1}}$. The following hold:  
\begin{align*}
 \E[ZZ^\top]&=\frac{\E[S^2]}d\Sigma, \\ 
 \calI(P_\Sigma)&=\frac1d\E\left[\left(\frac{f'(S)}{f(S)}\right)^2\right]
 \Sigma^{-1}. \numberthis  \label{eq:elliptical-second-moment}
\end{align*}
\end{lemma}
\begin{proof}
By \cref{lem:elliptical-radial-representation}, we have $Z=\Sigma^{1/2}SU$ with $U$ drawn uniformly from the unit sphere.  
Hence, $\E[U]=0$ and $\E[UU^\top]=I/d$. Independence between $S$ and $U$ gives
\begin{align*}
 \E[Z]&=\Sigma^{1/2}(\E[S])(\E[U])=0,\\
 \E[ZZ^\top]
 &=\Sigma^{1/2}\E[S^2]\E[UU^\top]\Sigma^{1/2}
 =\frac{\E[S^2]}d\Sigma.
\end{align*}
For $s=\norm z_{\Sigma^{-1}}>0$, chain rule with $P_\Sigma(z)\propto f(\|z\|_{\Sigma^{-1}})$ gives
\begin{align*}
 \nabla\log P_\Sigma(z)
 =\frac{f'(s)}{f(s)}\frac{\Sigma^{-1}z}{s}.
\end{align*}
At $Z=\Sigma^{1/2}SU$ this is
$(f'(S)/f(S))\Sigma^{-1/2}U$. Therefore
\begin{align*}
 \E[\nabla\log P_\Sigma(Z)\nabla\log P_\Sigma(Z)^\top]
 &=\E\left[\left(\frac{f'(S)}{f(S)}\right)^2\right]
   \Sigma^{-1/2}\E[UU^\top]\Sigma^{-1/2}\\
 &=\frac1d\E\left[\left(\frac{f'(S)}{f(S)}\right)^2\right]
   \Sigma^{-1}.
\end{align*}
\end{proof}

\begin{lemma}[Whitening reduction]\label{lem:elliptical-whitening-reduction}
Consider an elliptical family $P_\Sigma$ with a positive radial profile
and a score defined almost everywhere. For $\Sigma\succ0$ and $v\in\R^d$,
set $A=\Sigma^{1/2}$ and $h=A^{-1}v$. Then
$\norm h_2=\norm v_{\Sigma^{-1}}$, and
\begin{align*}
 \calI^v(P_\Sigma)&=A^{-1}\calI^h(P_I)A^{-1},\\
 \calJ^v(P_\Sigma)&=A^{-1}\calJ^h(P_I)A^{-1},\\ 
 \calJ^{v,\sfL}(P_\Sigma)&=A^{-1}\calJ^{h,\sfL}(P_I)A^{-1},\\
 \chi^2(P_\Sigma(\cdot-v)\|P_\Sigma)
 &=\chi^2(P_I(\cdot-h)\|P_I).
\end{align*}
\end{lemma}
\begin{proof}
Put $B=I+hh^\top$, so $\Sigma+vv^\top=ABA$.
By \cref{lem:elliptical-linear-transformation} and the chain rule,
\begin{align*}
 P_\Sigma(Ay+v)&=\frac{P_I(y+h)}{\det A},&
 P_{\Sigma+vv^\top}(Ay)&=\frac{P_B(y)}{\det A},\\
 \nabla P_\Sigma(Ay)&=\frac{A^{-1}\nabla P_I(y)}{\det A}.&&
\end{align*}
Substitute $z=Ay$ in the definitions of the information matrices and
chi-square divergence. The Jacobian $\d z=(\det A)\,\d y$ cancels
the determinant factors, and $\Sigma^{-1}z=A^{-1}y$ gives the two
congruence factors in the update matrices. Finally,
$\norm{Ay}_{\Sigma^{-1}}=\norm y_2$, so clipping becomes
$\norm y_2\le\sfL$. These substitutions prove all four identities.
\end{proof}
 
\begin{lemma}[Common radial information bounds]
\label{lem:radial-information-reduction}
Let $p(z)=P_I(z)\propto f(\|z\|_2)$ and have a score function $\nabla \log p(z)$ defined almost
everywhere, and let $\kappa=\norm v_2\le1/8$. Suppose that
$C_Z,C_s, C_D$ satisfy, for every $u\in\R^d$,
\begin{align}
 \E_{Z\sim p}[\inner{u,Z}^4]&\le C_Z\norm u_2^4, \qquad 
 \E_{Z\sim p}[\inner{u,\nabla\log p(Z)}^4]\le C_s\norm u_2^4, \qquad 
 \E_{Z\sim p}\left[\left(\frac{p(Z)}{p(Z+v)}\right)^2\right]\le C_D.
 \label{eq:radial-common-constants}
\end{align}
Then
\begin{align}
 1+\chi^2(p(\cdot-v)\|p)&\le\sqrt{C_D},
 \label{eq:common-chi-square}\\
 \calI^v(p)&\preceq\sqrt{C_sC_D}\,I,
 \label{eq:common-fisher-bound}
\end{align}
Furthermore, if $p(z)=c e^{-\psi(\norm z_2^2)}$ with $\psi$ concave on $[0,\infty)$, then 
\begin{align}
    \calJ^v(p)&\preceq2\sqrt{C_ZC_D}\,I.
 \label{eq:common-small-update-bound}
\end{align}
\end{lemma}
\begin{proof}
Write $s_p=\nabla\log p$, $p^+=P_{I+vv^\top}$, and
\begin{align*}
 D_v(z)=\frac{p(z)}{p(z+v)},\qquad
 R_v(z)=\frac{p^+(z)}{p(z)}.
\end{align*}
Translating the integration variable and applying Cauchy--Schwarz gives
\begin{align*}
 1+\chi^2(p(\cdot-v)\|p)
 &=\int\frac{p(z-v)^2}{p(z)}\,\d z
 =\int\frac{p(z)^2}{p(z+v)}\,\d z\\
 &=\E_p[D_v(Z)]\le\sqrt{\E_p[D_v(Z)^2]}\le\sqrt{C_D}.
\end{align*}
Similarly, substituting $\nabla p=ps_p$ in the shifted Fisher matrix gives, for any $u$, 
\begin{align*}
 u^\top\calI^v(p)u
 &=\E_p[\inner{u,s_p(Z)}^2D_v(Z)]\\
 &\le\sqrt{\E_p[\inner{u,s_p(Z)}^4]\E_p[D_v(Z)^2]}
 \le\sqrt{C_sC_D}\norm u_2^2.
\end{align*}
This proves \cref{eq:common-chi-square,eq:common-fisher-bound}.
 
For the elliptical information matrix, its definition and Cauchy-Schwarz gives
\begin{align}
 u^\top\calJ^v(p)u
 &=\E_p[\inner{u,Z}^2R_v(Z)^2D_v(Z)]\notag\\
 &\le\sqrt{C_D\E_p[\inner{u,Z}^4R_v(Z)^4]}.
 \label{eq:common-coupled-update}
\end{align}
Set $B=I+vv^\top$, $b=1+\kappa^2$, and
$Q=4B^{-1}-3I$. The rank-one inverse and determinant formulas give
\begin{align*}
 Q=I-\frac{4vv^\top}{b},\qquad
 \det Q= 1-\frac{4\kappa^2}{b}=\frac{1-3\kappa^2}{1+\kappa^2}\ge\frac{61}{65},\qquad
 Q^{-1}\preceq\frac{65}{61}I.
\end{align*}
In particular, $Q\succ0$. Since $\det B=b$, the updated
density is $p^+(z)=c b^{-1/2}e^{-\psi(z^\top B^{-1}z)}$ by \cref{lem:elliptical-linear-transformation}.
The identity
\begin{align*}
 z^\top B^{-1}z=\frac34\norm z_2^2+\frac14z^\top Qz
\end{align*}
and concavity imply
$4\psi(z^\top B^{-1}z)-3\psi(\norm z_2^2)\ge\psi(z^\top Qz)$.
Consequently,
\begin{align}
 p(z)R_v(z)^4
 &=\frac c{b^2}e^{-4\psi(z^\top B^{-1}z)+3\psi(\norm z_2^2)}
 \le\frac1{b^2}p(Q^{1/2}z).
 \label{eq:common-update-density}
\end{align}
Substituting $y=Q^{1/2}z$ therefore yields
\begin{align*}
 \E_p[\inner{u,Z}^4R_v(Z)^4]
 &\le\frac1{b^2\sqrt{\det Q}}
       \E_p[\inner{Q^{-1/2}u,Z}^4]\\
 &\le\frac{C_Z}{b^2\sqrt{\det Q}}(u^\top Q^{-1}u)^2.
\end{align*}
Combining this with \cref{eq:common-coupled-update} gives
\begin{equation}\label{eq:common-update-bound}
 \calJ^v(p)\preceq
 \frac{\sqrt{C_ZC_D}}{b(\det Q)^{1/4}}Q^{-1}
 \preceq2\sqrt{C_ZC_D}\,I,
\end{equation}
where we used $b\ge1$ and $(65/61)^{5/4}<2$.
 
\end{proof}
 
\begin{lemma}[Radial integration and fourth moments]
\label{lem:radial-fourth-moments}
Let $Z\sim p$, where $p(z)=cf(\norm z_2)$.
Write $Z=SU$ with $U$ uniform on the unit sphere and independent of $S$.
For every nonnegative measurable $H$, measurable scalar function $h$,
and $u\in\R^d$,
\begin{align}
 \E[H(S)]
 &=\frac{\int_0^\infty H(r)r^{d-1}f(r)\,\d r}
         {\int_0^\infty r^{d-1}f(r)\,\d r},
 \label{eq:common-radial-integration}\\
 \E[\inner{u,h(S)U}^4]
 &=\frac{3\E[h(S)^4]}{d(d+2)}\norm u_2^4.
 \label{eq:common-radial-fourth}
\end{align}
In particular, taking $h(r)=r$ gives 
\begin{align}
   \E[\inner{u,Z}^4] = \frac{3\E[S^4]}{d(d+2)}\|u\|_2^4.   \label{eq: position 4th}
\end{align}
If $f$ is
differentiable almost everywhere, taking $h(r)=f'(r)/f(r)$ gives
\begin{align}
   \E[\inner{u, \nabla \log p(Z)}^4] = \frac{3}{d(d+2)}\E\left[\left(\frac{f'(S)}{f(S)}\right)^4\right]\|u\|_2^4   \label{eq: score 4th}
\end{align}
since
$\nabla\log p(SU)=(f'(S)/f(S))U$ almost surely.
\end{lemma}
\begin{proof}
By \cref{lem:elliptical-radial-representation}, the radius $S$ has
density proportional to $r^{d-1}f(r)$ and is independent of $U$.
This proves \cref{eq:common-radial-integration}.
For the spherical fourth moment, first suppose $d\ge2$ and set
$a=\E[U_1^4]$, $\beta=\E[U_1^2U_2^2]$, where $U_i$ denotes the $i$-th coordinate of $U$.  Rotational invariance and
symmetry under sign changes give
\begin{align*}
 a=\E\left[\left(\frac{U_1+U_2}{\sqrt2}\right)^4\right]
 =\frac{2a+6\beta}{4},\qquad
 1=\E\left[\left(\sum_{i=1}^dU_i^2\right)^2\right]
 =da+d(d-1)\beta.
\end{align*}
Hence $a=3\beta=3/[d(d+2)]$. For $d=1$, the same formula follows
from $U\in\{-1,1\}$. By rotational invariance, 
\begin{align*}
 \E[\inner{u,U}^4]=\frac{3}{d(d+2)}\norm u_2^4.
\end{align*}
Independence of $S,U$ proves \cref{eq:common-radial-fourth}.
\end{proof}
 
\begin{lemma}[Density of a spherical projection]
\label{lem:spherical-projection-density}
Let $U$ be uniform on the Euclidean unit sphere $\mathbb S^{d-1}$.
For every nonzero $u\in\R^d$, set $X=\inner{u,U}/\norm u_2$.
If $d\ge2$, then $X$ has density
\begin{align*}
 q_d(x)=\frac{(1-x^2)^{(d-3)/2}}
 {\int_{-1}^1(1-t^2)^{(d-3)/2}\,\d t},\qquad -1<x<1,
\end{align*}
and zero density outside $(-1,1)$. If $d=1$, then $X$ is uniform on
$\{-1,1\}$.
\end{lemma}

\begin{proof}
Suppose $d\ge2$, and let $\theta\in(0,\pi)$ be the angle between $U$ and $u/\|u\|_2$, so that $X=\cos\theta$. The points at angle $\theta$ form a $(d-2)$-dimensional sphere of radius $\sin\theta$. A strip with angular width $d\theta$ has width $d\theta$, so its surface area is proportional to $(\sin\theta)^{d-2}\,d\theta$. Since $U$ is uniform, the density of $\theta$ is therefore proportional to $(\sin\theta)^{d-2}$.
Now set $x=\cos\theta$. Since $|dx/d\theta|=\sin\theta$, the change-of-variables formula gives
\[
q_d(x) \d x\propto (\sin\theta)^{d-2}\left|\frac{d\theta}{dx}\right| \d x = \frac{(\sin\theta)^{d-2}}{\sin\theta} \d x= (1-x^2)^{(d-3)/2} \d x, \qquad -1<x<1.
\]
Normalizing gives the stated density. For $d=1$, the unit sphere is $\{-1,1\}$, which gives the remaining claim directly.
\end{proof}
 
\subsection{Gaussian}\label[appendix]{app:gaussian-update-energy}
\begin{lemma}[Gaussian information]
\label{lem:gaussian-shifted-information}
\label{lem:gaussian-update-information}
Consider $P_\Sigma=\calN(0,\Sigma)$. Let $v$ be such that $\kappa=\|v\|_{\Sigma^{-1}}\le1/8$. Then 
\begin{align*}
 1+\chi^2(P_\Sigma(\cdot-v)\|P_\Sigma)&=\calO(1),\\
 \calI^v(P_\Sigma)&\preceq\calO(1)\Sigma^{-1}, \\ 
 \calJ^v(P_\Sigma)&\preceq\calO(1)\Sigma^{-1}.
\end{align*}
\end{lemma}
\begin{proof}
By \cref{lem:elliptical-whitening-reduction}, it suffices to take
$\Sigma=I$. Write $p=\calN(0,I)$ and $\kappa=\norm v_2$.
We calculate the constants in \cref{lem:radial-information-reduction}.
For a standard normal scalar $Y$ with density $\varphi$, integration
by parts and $\varphi'(y)=-y\varphi(y)$ give
\begin{align*}
 \E[Y^4]
 =-\int_{\R}y^3\varphi'(y)\,\d y
 =3\int_{\R}y^2\varphi(y)\,\d y=3,
\end{align*}
where the boundary term vanishes. Since
$\inner{u,Z}$ has the law of $\norm u_2Y$ and $\nabla \log p(Z)=-Z$, this gives
$C_Z=C_s=3$.
 
For the shift ratio,
$D_v(z)=\exp(\inner{v,z}+\kappa^2/2)$. Completing the square gives
\begin{align*}
 p(z)D_v(z)^2
 &=(2\pi)^{-d/2}
   \exp\left(-\tfrac12\norm z_2^2+2\inner{v,z}+\kappa^2\right)\\
 &=e^{3\kappa^2}p(z-2v).
\end{align*}
Integrating establishes the valid choices
\begin{equation}\label{eq:gaussian-common-constants}
 C_Z=C_s=3,\qquad C_D=e^{3\kappa^2}\le e^{3/64}.
\end{equation}
Also $p(z)=c e^{-\psi(\norm z_2^2)}$ with the concave function
$\psi(t)=t/2$. Applying
\cref{eq:common-chi-square,eq:common-fisher-bound,eq:common-small-update-bound}
therefore gives
\begin{align*}
 1+\chi^2(p(\cdot-v)\|p)&\le e^{3\kappa^2/2}=\calO(1),\\
 \calI^v(p)&\preceq\sqrt{3e^{3\kappa^2}}\,I=\calO(1)I,\\
 \calJ^v(p)&\preceq2\sqrt{3e^{3\kappa^2}}\,I=\calO(1)I.
\end{align*}
\end{proof}
 
\subsection{Radial Laplace}\label[appendix]{app:exponential-update-energy}
\begin{lemma}[Radial Laplace information]
\label{lem:laplace-shifted-information}
\label{lem:normalized-laplace-information}
Consider $P_\Sigma(z)\propto e^{-\sqrt d\norm z_{\Sigma^{-1}}}$. Let $v$ be such that 
$\kappa = \|v\|_{\Sigma^{-1}}\le1/8$. Then
\begin{align}
 1+\chi^2(P_\Sigma(\cdot-v)\|P_\Sigma)&=\calO(1),\\
 \calI^v(P_\Sigma)&\preceq\calO(1)\Sigma^{-1},
 \label{eq:laplace-small-shift}\\
 \calJ^v(P_\Sigma)&\preceq\calO(1)\Sigma^{-1}.
 \label{eq:normalized-laplace-small-shift}
\end{align}
\end{lemma}
\begin{proof}
By \cref{lem:elliptical-whitening-reduction}, it suffices to take
$\Sigma=I$. Write $p(z)=c e^{-\sqrt d\norm z_2}$ and
$\kappa=\norm v_2$. We calculate $C_D$, $C_s$, and $C_Z$ from
\cref{lem:radial-information-reduction}.
 
First, let $D_v(z)=p(z)/p(z+v)$. If $\kappa=0$, then
$\E_p[D_v(Z)^2]=1$. For $\kappa>0$, let $U$ be uniform on the unit
sphere, and set
$X=\inner{U,v/\kappa}$ and $\alpha=3\sqrt d\kappa$.
For $z\ne0$, convexity of the Euclidean norm gives
$\norm{z-v}_2\ge\norm z_2-\inner{z,v}/\norm z_2$.
Translating the integration variable and substituting the density yields
\begin{align*}
 \E_p[D_v(Z)^2]
 &=\int\frac{p(z)^3}{p(z+v)^2}\,\d z\\
 &=c\int e^{\sqrt d(2\norm z_2-3\norm{z-v}_2)}\,\d z\\
 &\le c\int e^{-\sqrt d\norm z_2+
                 3\sqrt d\inner{z,v}/\norm z_2}\,\d z\\
 &=\E[e^{3\sqrt d\inner{U,v}}]=\E[e^{\alpha X}].
\end{align*}
For $d\ge4$, put $n=(d-3)/2>0$. By
\cref{lem:spherical-projection-density}, $X$ has density proportional
to $(1-x^2)^n$ on $(-1,1)$, so
\begin{align*}
 \E[e^{\alpha X}]
 =\frac{\int_{-1}^1e^{\alpha x}(1-x^2)^n\,\d x}
        {\int_{-1}^1(1-x^2)^n\,\d x}.
\end{align*}
Using $\alpha x\le nx^2/2+\alpha^2/(2n)$,
$\log(1-x^2)\le-x^2$, and $\sqrt{1-x^2}\le1-x^2/2$, we obtain
\begin{align*}
 e^{\alpha x}(1-x^2)^n
 &\le e^{\alpha^2/(2n)}(1-x^2)^{n/2}
 \le e^{\alpha^2/(2n)}(1-x^2/2)^n,\\
 \int_{-1}^1(1-x^2/2)^n\,\d x
 &=\sqrt2\int_{-1/\sqrt2}^{1/\sqrt2}(1-y^2)^n\,\d y
 \le\sqrt2\int_{-1}^1(1-y^2)^n\,\d y.
\end{align*}
Consequently,
\begin{align*}
 \E[e^{\alpha X}]
 \le\sqrt2 e^{\alpha^2/(2n)}
 =\sqrt2\exp\left(\frac{9d\kappa^2}{d-3}\right)
 \le\sqrt2 e^{9/16}<3.
\end{align*}
For $d\le3$, $|X|\le1$ instead gives
$\E[e^{\alpha X}]\le e^\alpha\le e^{3/4}<3$. Thus, including
$\kappa=0$, a valid choice in every dimension is
\begin{equation}\label{eq:laplace-ratio-second-moment}
 \E_p[D_v(Z)^2]\le C_D=3.
\end{equation}
 
Next, $f'(r)/f(r)=-\sqrt d$. The radial fourth-moment formula
\cref{eq:common-radial-fourth} gives
\begin{align*}
 \E_p[\inner{u,\nabla \log p(Z)}^4]
 =\frac{3d}{d+2}\norm u_2^4
 \le3\norm u_2^4,
\end{align*}
so we may take $C_s=3$.
 
Finally, we calculate $C_Z$.
For $m_k=\int_0^\infty r^k e^{-\sqrt d\,r}\,\d r$, integration
by parts gives $m_k=k m_{k-1}/\sqrt d$ for $k\ge1$; the boundary
term vanishes. Therefore \cref{eq:common-radial-integration} gives
\begin{align*}
 \E[S]
 &=\frac{m_d}{m_{d-1}}=\sqrt d,\\
 \E[S^2]
 &=\frac{m_{d+1}}{m_{d-1}}=d+1,\\
 \E[S^4]
 &=\frac{m_{d+3}}{m_{d-1}}
 =\frac{d(d+1)(d+2)(d+3)}{d^2}.
\end{align*}
In particular, $\Cov(P_\Sigma)=(d+1)\Sigma/d$ by
\cref{lem:elliptical-moments}.
By \cref{eq: position 4th},
\begin{align*}
 \E_p[\inner{u,Z}^4]
 =\frac{3\E[S^4]}{d(d+2)}\norm u_2^4
 =\frac{3(d+1)(d+3)}{d^2}\norm u_2^4
 \le24\norm u_2^4.
\end{align*}
Thus $C_Z=24$ is valid. Since
$\psi(t)=\sqrt d\sqrt t$ is concave, applying
\cref{eq:common-chi-square,eq:common-fisher-bound,eq:common-small-update-bound}
gives
\begin{align*}
 1+\chi^2(p(\cdot-v)\|p)&\le\sqrt3,\\
 \calI^v(p)&\preceq3I,\\
 \calJ^v(p)&\preceq12\sqrt2\,I.
\end{align*}
\end{proof}
 
\subsection{Multivariate \texorpdfstring{Student-$t_2$}{Student-t2}}\label[appendix]{app:student-two}
\begin{lemma}[Student-$t_2$ shifted information]
\label{lem:student-two-shifted-information}
Consider $P_\Sigma(z)\propto(1+\norm z_{\Sigma^{-1}}^2)^{-(d+2)/2}$.
Let $v$ be such that $\kappa=\|v\|_{\Sigma^{-1}}\le1/8$. Then 
\begin{align}
 1+\chi^2(P_\Sigma(\cdot-v)\|P_\Sigma)&=\calO(1),
 \label{eq:student-two-small-shift-ratio}\\
 \calI^v(P_\Sigma)&\preceq\calO(1)\Sigma^{-1}.
 \label{eq:student-two-small-shift-fisher}
\end{align}
\end{lemma}
\begin{proof}
By \cref{lem:elliptical-whitening-reduction}, it suffices to take
$\Sigma=I$. Write $p(z)=c(1+\norm z_2^2)^{-\nu}$, where
$\nu=(d+2)/2$, and $\kappa=\norm v_2$. We calculate $C_D$ and $C_s$
from \cref{lem:radial-information-reduction}.
 
First, for $x,y>0$ with $3x-2y>0$, concavity of the logarithm gives
$3\log x\ge2\log y+\log(3x-2y)$. Taking
$x=1+\norm z_2^2$ and $y=1+\norm{z+v}_2^2$, we have
\begin{align*}
 3x-2y&=1-6\kappa^2+\norm{z-2v}_2^2>0,\\
 \frac{p(z)^3}{p(z+v)^2}
 &=c\left(\frac{y^2}{x^3}\right)^\nu
 \le c\bigl(1-6\kappa^2+\norm{z-2v}_2^2\bigr)^{-\nu}.
\end{align*}
Set $\alpha=1-6\kappa^2\in[29/32,1]$. Since
$\nu-d/2=1$, by \cref{lem:elliptical-linear-transformation} (with $B=I$ and $A=\sqrt{\alpha} I$), the density $P_{\alpha I}$ satisfies
\begin{align*}
    P_{\alpha I}(y) = \alpha^{-d/2} P_I(y/\sqrt{\alpha}) = \alpha^{-d/2} c(1+\|y\|_2^2/\alpha)^{-\nu}=c\alpha(\alpha+\norm y_2^2)^{-\nu}. 
\end{align*}
Thus the pointwise comparison above can be written as
\begin{equation}\label{eq:student-two-shift-density-comparison}
 \frac{p(z)^3}{p(z+v)^2}
 \le\frac1\alpha P_{\alpha I}(z-2v)
 \le\frac{32}{29}P_{\alpha I}(z-2v).
\end{equation}
Integrating gives
\begin{equation}\label{eq:student-two-ratio-second-moment}
 \E_p[D_v(Z)^2]
 \le\frac1\alpha\int P_{\alpha I}(z-2v)\,\d z
 =\frac1{1-6\kappa^2}\le\frac{32}{29}.
\end{equation}
Thus $C_D=32/29$ is valid.
 To calculate $C_s$, first note that
\[
\frac{\d}{\d r}\left(\frac{r}{\sqrt{1+r^2}}\right)^d
=d\,r^{d-1}(1+r^2)^{-\nu},
\]
so
\begin{equation}\label{eq:student-two-radial-normalizer}
\int_0^\infty r^{d-1}(1+r^2)^{-\nu}\,\d r
=\frac1d.
\end{equation}
Applying the radial integration formula
\cref{eq:common-radial-integration} with
$H(r)=r^4/(1+r^2)^4$ gives
\begin{align*}
\E\left[\frac{S^4}{(1+S^2)^4}\right]
&=\frac{\displaystyle\int_0^\infty
r^{d+3}(1+r^2)^{-(d+10)/2}\,\d r}
{\displaystyle\int_0^\infty
r^{d-1}(1+r^2)^{-\nu}\,\d r}\\
&=d\int_0^\infty
r^{d+3}(1+r^2)^{-(d+10)/2}\,\d r.
\end{align*}
Now substitute $x=r^2/(1+r^2)$, so that
\[
r^2=\frac{x}{1-x},\qquad
1+r^2=\frac1{1-x},\qquad
\d r=\frac{\d x}{2\sqrt{x}(1-x)^{3/2}}.
\]
Then
\begin{align*}
\E\left[\frac{S^4}{(1+S^2)^4}\right]
&=d\int_0^1
\left(\frac{x}{1-x}\right)^{(d+3)/2}
(1-x)^{(d+10)/2}
\frac{\d x}{2\sqrt{x}(1-x)^{3/2}}\\
&=\frac d2\int_0^1
x^{d/2+1}(1-x)^2\,\d x.
\end{align*}
Two integrations by parts evaluate the last integral:
\begin{align*}
 \int_0^1x^{d/2+1}(1-x)^2\,\d x
 &=\left[\frac{x^{d/2+2}(1-x)^2}{d/2+2}\right]_0^1
   +\frac2{d/2+2}\int_0^1x^{d/2+2}(1-x)\,\d x,\\
 \int_0^1x^{d/2+2}(1-x)\,\d x
 &=\left[\frac{x^{d/2+3}(1-x)}{d/2+3}\right]_0^1
   +\frac1{d/2+3}\int_0^1x^{d/2+3}\,\d x\\
 &=\frac1{(d/2+3)(d/2+4)}.
\end{align*}
Both boundary terms vanish. Multiplying by $d/2$ yields
\begin{align*}
 \E\left[\frac{S^4}{(1+S^2)^4}\right]
 =\frac{8d}{(d+4)(d+6)(d+8)}.
\end{align*}
Since $f'(r)/f(r)=-(d+2)r/(1+r^2)$,
\cref{eq: score 4th} now gives
\begin{align*}
 \E_p[\inner{u,\nabla \log p(Z)}^4]
 &=\frac{3(d+2)^4}{d(d+2)}
   \E\left[\frac{S^4}{(1+S^2)^4}\right]\norm u_2^4\\
 &=\frac{24(d+2)^3}{(d+4)(d+6)(d+8)}\norm u_2^4
 \le24\norm u_2^4.
\end{align*}
Thus $C_s=24$ is valid. Applying
\cref{eq:common-chi-square,eq:common-fisher-bound} proves
\begin{align*}
 1+\chi^2(p(\cdot-v)\|p)\le\sqrt{32/29},\qquad
 \calI^v(p)\preceq\sqrt{24\cdot32/29}\,I\preceq6I.
\end{align*}
 
\end{proof}
 
\begin{lemma}[Student-$t_2$ clipped update information]
\label{lem:student-two-clipped-update-information}
Consider $P_\Sigma(z)\propto(1+\norm z_{\Sigma^{-1}}^2)^{-(d+2)/2}$.
Let $v$ be such that $\kappa=\|v\|_{\Sigma^{-1}}\le1/8$. Then 
\begin{equation}\label{eq:student-two-small-shift-update}
 \calJ^{v,\sfL}(P_\Sigma)\preceq6\log(e\sfL)\Sigma^{-1}.
\end{equation}
Also, for $Z\sim P_\Sigma$ and $S=\norm Z_{\Sigma^{-1}}$,
\begin{align*}
 \E\left[S^2\mathbf1_{\{S\le\sfL\}}\right]
 \le\frac d2\log(1+\sfL^2).
\end{align*}
\end{lemma}
\begin{proof}
By \cref{lem:elliptical-whitening-reduction}, it suffices to take
$\Sigma=I$. Write $p(z)=c(1+\norm z_2^2)^{-\nu}$, where
$\nu=(d+2)/2$, and $p^+=P_{I+vv^\top}$.
Let $Z\sim p$ and $S=\norm Z_2$.
The radial integration formula in \cref{lem:radial-fourth-moments}
and the normalizer \cref{eq:student-two-radial-normalizer} give
$\E[S^4]=\infty$, since
$r^{d+3}(1+r^2)^{-\nu}\ge2^{-\nu}r$ for $r\ge1$.
Thus $C_Z=\infty$. Therefore, we turn to estimate the clipped integral.
 
First, the same radial integration formula gives
\begin{align*}
 \frac1d\E\left[S^2\mathbf1_{\{S\le\sfL\}}\right]
 &=\int_0^{\sfL}r^{d+1}(1+r^2)^{-(d+2)/2}\,\d r\\
 &=\int_0^{\sfL}\left(\frac r{\sqrt{1+r^2}}\right)^d
                  \frac r{1+r^2}\,\d r \\
 &\leq \int_{0,\sfL} \frac{r}{1+r^2}\d r 
 \le\frac12\log(1+\sfL^2).
\end{align*}
For the density comparison, set $B=I+vv^\top$, $b=1+\kappa^2$, and
\begin{align*}
 Q&=2B^{-1}-I = 2\left(I - \frac{vv^\top}{b}\right) - I =  I-\frac{2vv^\top}{b}, \\
 m&=Q^{-1}v = \left(I + \frac{\frac{2}{b}vv^\top}{1-\frac{2}{b}\|v\|_2^2}\right) v  = \frac b{1-\kappa^2}v, \numberthis \label{eq: Q inv}\\
 \alpha&=\frac{1-3\kappa^2}{1-\kappa^2}\geq \frac{1}{2},  \tag{$\kappa\leq \frac{1}{8}$}
\end{align*}
where we use the Sherman–Morrison formula $(I+uv^\top)^{-1} =  I - \frac{uv^\top}{1 + v^\top u}$. 
Completing the square yields
\begin{align*}
 2(1+z^\top B^{-1}z)-(1+\norm{z+v}_2^2)
 &=1+z^\top Qz-2\inner{v,z}-\kappa^2\\
 &=\alpha+(z-m)^\top Q(z-m)>0.
\end{align*}
For $x,y>0$ with $2x-y>0$, concavity of the logarithm gives
$2\log x\ge\log y+\log(2x-y)$. Applying this with
$x=1+z^\top B^{-1}z$ and $y=1+\norm{z+v}_2^2$ gives
\begin{align*}
 \frac{p^+(z)^2}{p(z+v)}
 &=\frac cb\left(\frac y{x^2}\right)^\nu
 \le\frac cb\bigl(\alpha+(z-m)^\top Q(z-m)\bigr)^{-\nu}. \numberthis \label{eq: tmptmp32}
 \end{align*}
Define $C=\alpha Q^{-1}$ and  $A=\frac1{\alpha b\sqrt{\det Q}}$. Then by \cref{lem:elliptical-linear-transformation}, 
\begin{align*}
P_C(y)=\frac{c}{\sqrt{\det C}}
\left(1+y^\top C^{-1}y\right)^{-\nu}, 
\end{align*}
and thus from \cref{eq: tmptmp32} we have 
\begin{align*}
  \frac{p^+(z)^2}{p(z+v)} \leq \frac{\sqrt{\det C}}{b}\alpha^{-\nu}P_C(z-m) =A P_C(z-m).
\end{align*}
The last equality uses $\det C=\alpha^d/\det Q$ and
$\nu=d/2+1$.
The eigenvalues of $C$ are $\alpha$ off the direction $v$ and
$\alpha b/(1-\kappa^2)=1-4\kappa^4/(1-\kappa^2)^2$ along $v$ (by \cref{eq: Q inv}).
Since $\kappa\le1/8$, we obtain
\begin{align*}
 \frac12I\preceq C\preceq I,\qquad
 \norm m_2\le\frac{65}{504}\le\frac14,\qquad
 A=\frac1{\alpha\sqrt{1-\kappa^4}}\le2.
\end{align*}
Let $Y\sim p$. By \cref{lem:elliptical-linear-transformation},
$m+C^{1/2}Y$ has density $P_C(\cdot-m)$. If
$\norm{m+C^{1/2}Y}_2\le\sfL$, then
\begin{align*}
 \norm Y_2
 &\le\norm{C^{-1/2}}_{\op}(\sfL+\norm m_2)
 \le\sqrt2(\sfL+1/4)\le2\sfL\qquad(\sfL\ge1).
\end{align*}
Using the density comparison and enlarging to this centered ball gives
\begin{align*}
 \calJ^{v,\sfL}(p)
 &=  \int_{\|z\|_2 \leq \sfL}  \frac{p^+(z)^2}{p(z+v)} zz^\top \d z\\ 
 &\preceq A \int_{\|z\|_2\leq \sfL }P_C(z-m) zz^\top \d z \\ 
 &\preceq A\E[(m+C^{1/2}Y)(m+C^{1/2}Y)^\top
                    \mathbf1_{\{\norm Y_2\le2\sfL\}}]\\
 &=A\left(mm^\top\Pr(\norm Y_2\le2\sfL)
       +\frac1d\E\left[S^2\mathbf1_{\{S\le2\sfL\}}\right]C\right)\\
 &\preceq2\left(\frac1{16}
       +\frac1d\E\left[S^2\mathbf1_{\{S\le2\sfL\}}\right]\right)I\\
 &\preceq\bigl(1/8+\log(1+4\sfL^2)\bigr)I
 \preceq2\log(e\sfL)I.
\end{align*}
The last equality uses symmetry to eliminate the cross terms and the radial
identity
\begin{align*}
 \E\left[YY^\top\mathbf1_{\{\norm Y_2\le2\sfL\}}\right]
 =\frac1d\E\left[S^2\mathbf1_{\{S\le2\sfL\}}\right]I.
\end{align*}
The last step uses $1+4\sfL^2\le5\sfL^2$ and $1/8+\log5<2$.
Since $2\log(e\sfL)\le6\log(e\sfL)$, this proves \cref{eq:student-two-small-shift-update}.

\end{proof}
 
\begin{lemma}[Student-$t_2$ radial tail]
\label{lem:student-two-radial-tail}
For the Student-$t_2$ family and $\sfL\ge1$,
let $Z\sim P_\Sigma$ and $S=\norm Z_{\Sigma^{-1}}$. Then
\begin{align*}
 \E\left[S\mathbf1_{\{S>\sfL\}}\right]
 \le\frac d{\sqrt{1+\sfL^2}}.
\end{align*}
\end{lemma}
\begin{proof}
By \cref{lem:elliptical-whitening-reduction}, it suffices to take
$\Sigma=I$. Apply the radial integration formula from
\cref{lem:radial-fourth-moments} with
$H(s)=s\mathbf1_{\{s>\sfL\}}$. The normalizer in
\cref{eq:student-two-radial-normalizer} gives
\begin{align*}
 \E\left[S\mathbf1_{\{S>\sfL\}}\right]
 &=d\int_{\sfL}^\infty s^d(1+s^2)^{-(d+2)/2}\,\d s\\
 &=d\int_{\sfL}^\infty
       \left(\frac s{\sqrt{1+s^2}}\right)^{d-1}
       \frac{s}{(1+s^2)^{3/2}}\,\d s\\
 &\le d\int_{\sfL}^\infty\frac{s}{(1+s^2)^{3/2}}\,\d s
 =d\left[-\frac1{\sqrt{1+s^2}}\right]_{\sfL}^\infty
 =\frac d{\sqrt{1+\sfL^2}}.
\end{align*}
The inequality follows from $0\le s/\sqrt{1+s^2}\le1$ and $d\ge1$.
This proves the radial-tail entry in \cref{tab:elliptical-updates}.
\end{proof}
\section{Geometric-Resampling Proof}\label[appendix]{app:gr}
 
The proof of \cref{thm:estimator-bounds} combines the covariance and
paired-observation identities with the finite-random-inverse bounds below.
 
\subsection{Covariance and the paired observation}
Throughout this appendix, work in normalized coordinates with
$\mu(\nu)=0$, $\Cov(\nu)=I/d$, and
$rB_2^d\subseteq \conv(\calA)\subseteq RB_2^d$. Recall that
$\tildep_t=(1-\gamma)p_t+\gamma\nu$, $\mu_t=\mu(p_t)$,
$M_t=\Cov(p_t)$, and $\tildeM_t=\Cov(\tildep_t)$.
The per-round calculations below condition on the history before round $t$.
 
\begin{lemma}[Exploration covariance and reward-vector norm]
\label{lem:gr-covariance}
For $\gamma\in(0,1)$,
\begin{align*}
 \tildeM_t
 =(1-\gamma)M_t+\frac\gamma d I
       +\gamma(1-\gamma)\mu_t\mu_t^\top, \qquad 
 \frac\gamma d I\preceq\tildeM_t\preceq R^2I,
 \qquad (1-\gamma)M_t\preceq\tildeM_t.
\end{align*}
Moreover, $\norm{\theta_t}_2\le1/r$.
\end{lemma}
\begin{proof}
The mixture covariance formula gives the displayed identity and both
lower bounds. For any unit vector $u$,
$u^\top\tildeM_tu\le\E_{a\sim\tildep_t}[\inner{u,a}^2]\le R^2$,
which proves the upper bound.
Since $r_t$ is affine, the reward bound on $\calA$ also holds on $\conv(\calA)$. Thus, for every unit vector $u$, the rewards of $ru$ and $-ru$ lie in $[-1,1]$. Their difference is $2r\langle u,\theta_t\rangle$, giving $\|\theta_t\|_2\le 1/r$.
\end{proof}
 
\begin{lemma}[Paired observation and exact estimator]
\label{lem:gr-paired-observation}
Let $a_t,\bar a_t$ be independent samples from $\tildep_t$, and let
$y_t=\inner{a_t,\theta_t}+b_t$. Then
\begin{align*}
 \E_t[(a_t-\bar a_t)y_t]&=\tildeM_t\theta_t,\\
 \E_t[(a_t-\bar a_t)(a_t-\bar a_t)^\top y_t^2]
 &\preceq2\tildeM_t.
\end{align*}
Consequently, the exact estimator
$\thetahat_t^{\rm ex}=\tildeM_t^{-1}(a_t-\bar a_t)y_t$ satisfies
\begin{align*}
 \E_t[\thetahat_t^{\rm ex}]&=\theta_t,
 &\E_t[\thetahat_t^{\rm ex}(\thetahat_t^{\rm ex})^\top]
 &\preceq2\tildeM_t^{-1},\\
 \E_t[(\thetahat_t^{\rm ex})^\top M_t\thetahat_t^{\rm ex}]
 &\le\frac{2d}{1-\gamma}.
\end{align*}
\end{lemma}
\begin{proof}
Write $\widetilde{\mu}_t=\mu(\tildep_t)$. Independence gives
\begin{align*}
 \E_t[(a_t-\bar a_t)y_t]
 &=\bigl(\E_{a\sim\tildep_t}[aa^\top]
       -\widetilde{\mu}_t\widetilde{\mu}_t^\top\bigr)\theta_t
       +b_t(\widetilde{\mu}_t-\widetilde{\mu}_t)\\
 &=\tildeM_t\theta_t.
\end{align*}
Thus pairing removes the intercept without requiring the reward of
$\bar a_t$. For the second moment, $|y_t|\le1$ implies
\begin{align*}
 \E_t[(a_t-\bar a_t)(a_t-\bar a_t)^\top y_t^2]
 \preceq\E_t[(a_t-\bar a_t)(a_t-\bar a_t)^\top]=2\tildeM_t.
\end{align*}
Multiplying by $\tildeM_t^{-1}$ proves the exact estimator's mean and
matrix second-moment bounds. Finally, \cref{lem:gr-covariance} gives
\begin{align*}
 \E_t[(\thetahat_t^{\rm ex})^\top M_t\thetahat_t^{\rm ex}]
 &=\tr\bigl(M_t\E_t[\thetahat_t^{\rm ex}(\thetahat_t^{\rm ex})^\top]\bigr)\\
 &\le2\tr(M_t\tildeM_t^{-1})\le\frac{2d}{1-\gamma}.
\end{align*}
\end{proof}
 
The last bound explains why we retain $M_t$ in the stability analysis:
the action covariance cancels the inverse covariance in the estimator's
second moment.
 
\subsection{The finite random inverse}
Fix $0<\beta\le1/(2R^2)$ and an integer $N\ge0$. Draw $N$ independent
pairs $u_i,v_i\sim\tildep_t$, independently of $a_t,\bar a_t$.
Define
\begin{align*}
 z_i&=(u_i-v_i)/\sqrt2,&
 Q_i&=I-\beta z_iz_i^\top,\\
 H_t&=\beta\sum_{k=0}^N Q_1\cdots Q_k,&
 \thetahat_t&=H_t(a_t-\bar a_t)y_t.\tag{GR}\label{eq:gr}
\end{align*}
The empty product is $I$. The construction is motivated by the matrix
geometric series
\begin{align*}
 \tildeM_t^{-1}=\beta\sum_{k=0}^\infty(I-\beta\tildeM_t)^k.
\end{align*}
An independent pair gives
$\E_t[z_iz_i^\top]=\tildeM_t$, so each random factor has mean
$I-\beta\tildeM_t$. Independence makes the mean of its ordered product
the corresponding matrix power. Truncating at $N$ gives a finite number
of oracle calls, at the price of a bias.
 
Fix a round $t$ and set $C=I-\beta\tildeM_t$.
 
\begin{lemma}[Mean and range of the random inverse]
\label{lem:gr-inverse-mean-range}
For the construction in \cref{eq:gr},
\begin{align*}
 \E_t[H_t]&=\tildeM_t^{-1}(I-C^{N+1}),
 &\norm{H_t}_{\op}&\le\beta(N+1).
\end{align*}
\end{lemma}
\begin{proof}
Since the two actions in each resampling pair are independent,
\begin{align*}
 \E_t[z_i z_i^\top]=\tildeM_t,\qquad
 \norm{z_i}_2^2\le2R^2,\qquad
 \beta\norm{z_i}_2^2\le1.
\end{align*}
Thus each factor $Q_i=I-\beta z_iz_i^\top$ is a positive semidefinite
contraction and has mean $C$. Independence between factors gives
\begin{align*}
 \E_t[H_t]=\beta\sum_{k=0}^NC^k=\tildeM_t^{-1}(I-C^{N+1}).
\end{align*}
Every ordered product of the $Q_i$ has operator norm at most one, even
though the product itself need not be symmetric. Summing the $N+1$
products therefore gives $\norm{H_t}_{\op}\le\beta(N+1)$.
\end{proof}
 
\begin{lemma}[Second moment of the random inverse]
\label{lem:gr-inverse-second-moment}
For the construction in \cref{eq:gr},
\begin{align*}
 \E_t\bigl[\tr(\tildeM_tH_t\tildeM_tH_t^\top)\bigr]\le2d.
\end{align*}
\end{lemma}
\begin{proof}
Write
\begin{align*}
 V_0=I,\qquad V_i=Q_1\cdots Q_i,\qquad
 H_t=\beta\sum_{i=0}^N V_i,\qquad C=I-\beta\tildeM_t.
\end{align*}
For the analysis, extend the sequence by one independent factor $Q_{N+1}$.
First, for $j\ge i$, conditioning on $Q_1,\ldots,Q_i$ gives
\begin{align*}
 \E_t\bigl[\tr(\tildeM_t V_i\tildeM_t V_j^\top)\bigr]
 =\E_t\bigl[\tr(\tildeM_t V_i C^{j-i}\tildeM_t V_i^\top)\bigr].
\end{align*}
Because
\begin{align*}
 \sum_{r=0}^{N-i}C^r\tildeM_t
 =\frac{I-C^{N-i+1}}{\beta}\preceq\frac{I}{\beta},
\end{align*}
expanding the square and pairing its off-diagonal terms yields
\begin{align}
 \E_t\bigl[\tr(\tildeM_t H_t\tildeM_t H_t^\top)\bigr]
 &\le2\beta^2\sum_{i=0}^N\sum_{j=i}^N
       \E_t\bigl[\tr(\tildeM_t V_i\tildeM_t V_j^\top)\bigr]\notag\\
 &\le2\beta\sum_{i=0}^N
       \E_t\bigl[\tr(\tildeM_t V_iV_i^\top)\bigr].
 \label{eq:gr-inverse-product-sum}
\end{align}
The first inequality counts the nonnegative diagonal terms twice.
 
Next, since $0\preceq Q_i\preceq I$,
\begin{align*}
 \E_t[Q_i^2]\preceq\E_t[Q_i]=I-\beta\tildeM_t.
\end{align*}
Writing $V_{i+1}=Q_1(Q_2\cdots Q_{i+1})$, the suffix is independent of
$Q_1$ and has the same distribution as $V_i$. Therefore, with
$\norm{\cdot}_F$ denoting the Frobenius norm,
\begin{align}
 \E_t\norm{V_{i+1}}_F^2
 &=\tr\bigl(\E_t[Q_1^2]\,\E_t[V_iV_i^\top]\bigr)\notag\\
 &\le\E_t\norm{V_i}_F^2
       -\beta\E_t\bigl[\tr(\tildeM_t V_iV_i^\top)\bigr].
 \label{eq:gr-inverse-product-decay}
\end{align}
Substituting \cref{eq:gr-inverse-product-decay} into
\cref{eq:gr-inverse-product-sum} telescopes:
\begin{align*}
 \E_t\bigl[\tr(\tildeM_t H_t\tildeM_t H_t^\top)\bigr]
 &\le2\sum_{i=0}^N\bigl(
       \E_t\norm{V_i}_F^2-\E_t\norm{V_{i+1}}_F^2\bigr)\\
 &=2\bigl(d-\E_t\norm{V_{N+1}}_F^2\bigr)\le2d.
\end{align*}
\end{proof}
 
\subsection{Proofs of the estimator guarantees}
\begin{proof}[Proof of \cref{thm:estimator-bounds}]
Fix a round $t$. The random inverse $H_t$ is conditionally independent
of $(a_t,\bar a_t,y_t)$. Thus
\cref{lem:gr-paired-observation,lem:gr-inverse-mean-range} give
\begin{align*}
 \E_t[\thetahat_t]
 &=\E_t[H_t]\,\E_t[(a_t-\bar a_t)y_t]\\
 &=\tildeM_t^{-1}(I-C^{N+1})\tildeM_t\theta_t
 =(I-C^{N+1})\theta_t,
\end{align*}
where $C$ commutes with $\tildeM_t$. By \cref{lem:gr-covariance},
$0\preceq C\preceq(1-\beta\gamma/d)I$ and $\norm{\theta_t}_2\le1/r$.
Therefore
\begin{align*}
 \norm{\E_t[\thetahat_t]-\theta_t}_2
 \le(1-\beta\gamma/d)^{N+1}\norm{\theta_t}_2
 \le e^{-\beta\gamma(N+1)/d}/r.
\end{align*}
The range bound in \cref{lem:gr-inverse-mean-range}, together with
$\norm{a_t-\bar a_t}_2\le2R$ and $|y_t|\le1$, gives
\begin{align*}
 \norm{\thetahat_t}_2
 \le\norm{H_t}_{\op}\norm{a_t-\bar a_t}_2|y_t|
 \le2R\beta(N+1).
\end{align*}
 
For the second moment, independence and
\cref{lem:gr-paired-observation} imply
\begin{align*}
 \E_t[\thetahat_t\thetahat_t^\top]
 &=\E_t\bigl[H_t\E_t[(a_t-\bar a_t)(a_t-\bar a_t)^\top y_t^2]H_t^\top\bigr]\\
 &\preceq2\E_t[H_t\tildeM_tH_t^\top].
\end{align*}
Taking the trace against $\tildeM_t$ and using
\cref{lem:gr-inverse-second-moment}, we obtain
\begin{align*}
 \E_t[\thetahat_t^\top\tildeM_t\thetahat_t]
 \le2\E_t\bigl[\tr(\tildeM_tH_t\tildeM_tH_t^\top)\bigr]\le4d.
\end{align*}
Since $M_t\preceq\tildeM_t/(1-\gamma)$ by
\cref{lem:gr-covariance}, this also gives
\begin{align*}
 \E_t[\thetahat_t^\top M_t\thetahat_t]\le\frac{4d}{1-\gamma}.
\end{align*}
\end{proof}

\section{Proof of the Final Regret Bound}\label[appendix]{app:final-proof}

\begin{lemma}[Initial potential]\label{lem:initial-potential}
Suppose $\calA\subseteq RB_2^d$ and $\lambda>0$. Let
$\rho_1=P_{\lambda I}$ be any of the Gaussian, radial-Laplace, or
Student-$t_2$ laws in
\cref{tab:convolution-initializers,tab:elliptical-updates}. Then
\begin{align*}
 \Psi_1(0)\le2R\sqrt{d\lambda}.
\end{align*}
\end{lemma}
\begin{proof}
Since $h(z)=\max_{a\in\calA}\inner{a,z}\le R\norm z_2$, it suffices
to show
\begin{equation}\label{eq:initializer-first-moments}
 \E_{Z\sim\rho_1}\left[\norm Z_2\right]\le2\sqrt{d\lambda}.
\end{equation}
For Gaussian,
$\E[ZZ^\top]=\lambda I$, so Cauchy--Schwarz gives
\begin{align*}
 \E\left[\norm Z_2\right]\le\sqrt{\E\left[\norm Z_2^2\right]}=\sqrt{d\lambda}.
\end{align*}
For radial Laplace, let $S=\norm Z_2/\sqrt\lambda$.
The radial moment calculation in the proof of
\cref{lem:laplace-shifted-information} gives $\E[S]=\sqrt d$, so
\begin{align*}
 \E\left[\norm Z_2\right]=\sqrt\lambda\,\E[S]=\sqrt{d\lambda}.
\end{align*}
For Student-$t_2$, let $S=\norm Z_2/\sqrt\lambda$.
Splitting the expectation at $\sqrt d$ and applying
\cref{lem:student-two-radial-tail} gives
\begin{align*}
 \E [S]
 &\le\sqrt d+\E\left[S\mathbf1_{\{S>\sqrt d\}}\right]
 \le\sqrt d+\frac d{\sqrt{1+d}}
 \le2\sqrt d.
\end{align*}
Multiplying by $\sqrt\lambda$ proves \cref{eq:initializer-first-moments}
for this family as well, and hence proves the initial-potential bound.
\end{proof}

\begin{proof}[Proof of \cref{cor:final-multiplier-bounds}]
If the proposed $\gamma>1/2$, bounded rewards give
\begin{align*}
 \Reg_T(a)\le2T<16R(L+1)d^{3/4}\sqrt T,
\end{align*}
which satisfies \cref{eq:final-regret-bound}. Henceforth assume
$\gamma\le1/2$. In either construction, write
$\Sigma_t=\lambda I+\sum_{\tau<t}\thetahat_\tau\thetahat_\tau^\top$.
The scale choice ensures
$\norm{\thetahat_t}_{\Sigma_t^{-1}}^2\le G^2/\lambda\le1/64$.
The elliptical potential lemma gives
\begin{align*}
 \sum_{t=1}^T\norm{\thetahat_t}_{\Sigma_t^{-1}}^2
 &\le \calO\left(\log\frac{\det\Sigma_{T+1}}{\det\Sigma_1}\right) \le \calO\left(d\log\left(1+\frac{TG^2}{d\lambda}\right)\right)
 =\widetilde{\calO}(d).
\end{align*}
The scale in \cref{eq:final-scale-choices} satisfies the hypotheses of
\cref{cor:A-multiplier,cor:B-multiplier,cor:A-elliptical,cor:B-elliptical}.
Applying those corollaries and summing their $T^{-1}$ remainders yields
\begin{align*}
 \E\left[\sum_{t=1}^T(A_t+B_t)\right]
 &\lesssim \sqrt{\log(dRGT)}
 \sqrt{\E\left[\sum_{t=1}^T\thetahat_t^\top M_t\thetahat_t\right]}
 \sqrt{\E\left[\sum_{t=1}^T\thetahat_t^\top\Sigma_t^{-1}\thetahat_t\right]}+1\\
 &\le\widetilde{\calO}(\sqrt{dT}\sqrt d)
 =\widetilde{\calO}(d\sqrt T).
\end{align*}
 
By \cref{lem:initial-potential}, $\Psi_1(0)\le2R\sqrt{d\lambda}$.
Combining \cref{eq:general-regret} with the stability and bias bounds gives
\begin{align*}
 \Reg_T(a)
 &\le2R\sqrt{d\lambda}
       +\widetilde{\calO}(d\sqrt T)+\tfrac12+2\gamma T \le\widetilde{\calO}\left(
       d\sqrt T+\frac{R^2d^{3/2}}{\gamma}+\gamma T\right).
\end{align*}
Substituting \cref{eq:parameters} proves \cref{eq:final-regret-bound}.
\end{proof}

\section{Computational Complexity Comparison}\label[appendix]{app:computational-comparison}
 
We compare the computational complexity of existing oracle-efficient adversarial linear bandit algorithms. This includes our \cref{alg:ftpl}, GeometricHedge with volumetric-spanner
exploration \citep{hazan2016volumetric}, and the algorithm
of \citet{ito2019oracle}.

In the \cref{tab:computational-comparison} below, we present the number of oracle calls of $\calA$ and arithmetic operations besides these oracle calls. These numbers are the total complexity over $T$ rounds. 
Both \citet{hazan2016volumetric} and \citet{ito2019oracle} involve sampling from a log-concave distribution, and we equip them with the state-of-the-art log-concave sampling procedure by \cite{liu2026near}. Hence, the computational complexity reported below is actually better than reported in their paper. 
The detailed calculation of the table entries are provided in the following subsections. 

The comparison in \cref{tab:computational-comparison} concerns convex action sets $\calA$. For nonconvex action sets, our FTPL approach works without modification since a linear optimization oracle for $\calA$ is equivalent to a linear optimization oracle for $\conv(\calA)$. In contrast, the approaches of \citet{hazan2016volumetric} and \citet{ito2019oracle} involve two additional computational procedures:  first, converting a linear optimization oracle for $\calA$ into a membership oracle or a separation oracle for $\conv(\calA)$; second, decomposing points in $\conv(\calA)$ into convex combinations of feasible actions in $\calA$. Both introduce overhead.

 
 \begin{table}[ht]
\centering
\small
\setlength{\tabcolsep}{3pt}
\renewcommand{\arraystretch}{1.5}
\resizebox{\linewidth}{!}{%
\begin{tabular}{|l|c|c|c|c|c|}
\hline
\multirow{2}{*}{Method}
  & \multicolumn{3}{c|}{Oracle calls to $\calA$}
  & \multirow{2}{*}{\makecell{Arithmetic\\outside oracles}}
  & \multirow{2}{*}{Regret}\\
\cline{2-4}
  & \makecell{Linear\\optimization}
  & Membership
  & Separation
  & & \\
\hline
Our elliptical FTPL
  & $d^2+d^{3/4}T^{3/2}$
  & $0$
  & $0$
  & $d^4+d^{11/4}T^{3/2}$
  & $d^{5/4}\sqrt{T}$\\
\hline
\makecell[l]{GeometricHedge with\\
volumetric spanner\\
\citep{hazan2016volumetric}\\
+ \citet{liu2026near}'s sampling}
  & $d^2$
  & $d^2T^{5/2}$
  & $0$
  & $d^4T^{5/2}$
  & $d\sqrt{T}$\\
\hline
\makecell[l]{\cite{ito2019oracle}\\
+ \citet{liu2026near}'s sampling}
  & $0$
  & $0$
  & $dT$
  & $d^5T^3+d^4T^4$
  & $d^{3/2}\sqrt{T}$\\
\hline
\end{tabular}%
}
\caption{Comparison on the oracle complexity and the computational complexity beyond oracle calls for three linear bandit algorithms over convex action set $\calA$. The complexities of \cite{hazan2016volumetric} and \cite{ito2019oracle} are improved from the original work through the use of the state-of-the-art log-concave sampling procedure \citep{liu2026near}. For example, the membership oracle complexity of \cite{hazan2016volumetric} is of order $d^3T^{7}$ in their work, but improved to $d^2 T^{5/2}$ in our table. }
\label{tab:computational-comparison}
\end{table}
 
 
\subsection{Our algorithm}
 
Each round of \cref{alg:ftpl} calls \texttt{SampleAction} $2N+2$
times, each requiring one linear optimization oracle call. Generating an elliptical
perturbation and updating its Cholesky factor each cost $\calO(d^2)$.
The resampling calculation also costs $\calO(d^2)$ per sample because
it is a rank-one update:
\[
 V_i=V_{i-1}-\beta(V_{i-1}z_i)z_i^\top,
 \qquad H_t=\beta\sum_{i=0}^N V_i.
\]
Thus the online costs are $\calO(T(N+1))$ linear optimization oracle calls and
$\calO(d^2T(N+1))$ arithmetic operations. By \cref{eq:parameters},
\[
 N+1=\left\lceil\frac{L}{2(L+1)}R d^{1/4}\sqrt T\right\rceil
     =\calO(d^{3/4}\sqrt T)
 \qquad\text{when }R=4\sqrt d.
\]
The barycentric-spanner preprocessing in
\cref{thm:spanner-normalization} adds
$\widetilde{\calO}(d^2)$ linear optimization oracle calls and
$\widetilde{\calO}(d^4)$ arithmetic operations. These give the first
row of the table.

\subsection{Estimating mean and covariance of a log-concave distribution}
 
Both \cite{hazan2016volumetric} and \cite{ito2019oracle} rely on sampling from log-concave distributions. We first specify what must be estimated, then
count the samples and the cost of generating them.
For a target distribution $p(a)\propto \exp(\theta^\top a)\mathbf{1}_{\{a\in\calA\}}$, we aim to produce approximation of the mean and covariance such that 
\begin{align*}
    \|\widehat{\mu} - \mu(p)\|_{\mathrm{Cov}(p)^{-1}} \leq \varepsilon, \qquad (1-\varepsilon)\mathrm{Cov}(p) \preceq \widehat{M} \preceq (1+\varepsilon)\mathrm{Cov}(p). \numberthis \label{eq: estimation error 1}
\end{align*}

By Theorem 4.1 of \citet{liu2026near}, a log-concave distribution can be transformed into near-isotropic position using $\widetilde{\calO}(d^3)$ potential evaluations. In these coordinates, Theorem C.25 provides samples at a cost of $\widetilde{\mathcal O}(d^2)$ unnormalized density evaluations (i.e., evaluate $\exp(\theta^\top a)\mathbf{1}_{\{a\in\calA\}}$) per sample, with initialization handled by Proposition A.9. The concentration argument in Lemma A.10 (also used in Lemma 4.3), with accuracy parameter $\varepsilon$, shows that $\widetilde{\mathcal O}(d\varepsilon^{-2})$ samples suffice to estimate the mean and covariance to the accuracy required in \cref{eq: estimation error 1}. Thus, since each unnormalized density evaluation for $p_\theta$ requires one membership-oracle call to $\mathcal A$, these results imply a total complexity of $\widetilde{\mathcal O}(d^3\varepsilon^{-2})$ membership-oracle calls.

\subsection{GeometricHedge with volumetric spanner and \cite{liu2026near}'s sampling procedure}
 
The algorithm of \cite{hazan2016volumetric} samples from $p_t(a) \propto \exp(\eta \Thetahat_t^\top a)\mathbf{1}_{\{a\in\calA\}}$ with probability $1-\gamma$ and from a volumetric spanner with probability $\gamma$ in every round. 
By \citet[Theorem~5]{hazan2016volumetric},
$\calO(d+\log^2 T)$ independent samples from $p_t$ suffice to construct the volumetric spanner in each round. 
 
Let $\varepsilon$ be the error specified in \cref{eq: estimation error 1}. 
Approximate inversion of the covariance matrix introduces a bias in $\widehat\theta_t$.
The calculation below bounds its effect on any action's expected
estimated loss by
$\calO(\varepsilon d/\sqrt{\gamma})$ per round. We choose $\varepsilon$ so
that this error is at most order $\gamma$, the exploration cost
already present in the regret bound. Thus take
\begin{equation}\label{eq:complexity-hk-alpha}
 \varepsilon=c\frac{\gamma^{3/2}}{d},
 \qquad \eta\asymp\sqrt{\frac{\log(2T)}{T}},
 \qquad \gamma=d\eta,
\end{equation}
with a sufficiently small universal constant $c$. Substitution gives
\[
 \varepsilon^{-2}
 =\calO\!\left(\frac{d^2}{\gamma^3}\right)
 =\widetilde{\calO}(T^{3/2}/d). 
\]
Thus, the number of membership oracle calls per round is
$\widetilde{\calO}(d^3\varepsilon^{-2})
=\widetilde{\calO}(d^2T^{3/2})$. Multiplying by $T$ rounds gives
$\widetilde{\calO}(d^2T^{5/2})$ membership queries. Multiplying
by the $\calO(d^2)$ arithmetic cost per evaluation gives
$\widetilde{\calO}(d^4T^{5/2})$ arithmetic operations.
 
Before the first round, an affine barycentric spanner costs
$\widetilde{\calO}(d^2)$ linear optimization oracle calls and
$\widetilde{\calO}(d^4)$ arithmetic operations. It supplies an
interior ball and an enclosing radius with ratio $\calO(d^2)$
in suitable coordinates. Adding up these costs give the second row of the table. 
 
\paragraph{Regret overhead due to estimation error.}
For $\varepsilon\le1/4$, the relative error in $\widehat M_t$ implies
\[
 (1-\varepsilon)\widetilde{M}_t\preceq\widehat M_t\preceq(1+\varepsilon)\widetilde{M}_t,
 \qquad
 \|\widetilde{M}_t^{1/2}\widehat M_t^{-1}\widetilde{M}_t^{1/2}-I\|_{\rm op}
 \le2\varepsilon, 
\]
where $\widetilde{M}_t$ is the covariance matrix under the mixture of $p_t$ and volumetric spanner with weights $1-\gamma$ and $\gamma$, respectively. 
Conditioning on the history, we have 
$\E_{a_t}[\thetahat_t]=\widehat M_t^{-1}\widetilde{M}_t\theta_t$.
The loss bound gives $\|\theta_t\|_{\widetilde{M}_t}\le1$, and the fixed spanner
gives $\max_{a\in\calA}\|a\|_{\widetilde{M}_t^{-1}}
=\calO(d/\sqrt{\gamma})$. Therefore
\begin{equation}\label{eq:complexity-hk-bias}
 \max_{a\in\calA}|a^\top(\E_{a_t}[\thetahat_t]-\theta_t)|
 \le2\varepsilon\max_{a\in\calA}\|a\|_{\widetilde{M}_t^{-1}}\|\theta_t\|_{\widetilde{M}_t}
 =\calO(\varepsilon d/\sqrt{\gamma}).
\end{equation}
 
 
\subsection{\cite{ito2019oracle}'s algorithm}
 
The algorithm of \citet{ito2019oracle} maintains an outer
approximation $D_{t,j}\supseteq \calA$ in stage $j$ of round $t$, and refine $D_{t,j}$ over time and stages by adding cuts. At the stage $j$ of round $t$, it estimates the
mean and covariance of the log-concave distribution $p_{t,j}(a)\propto \exp(\eta\Thetahat_t a) \mathbf{1}_{\{a\in D_{t,j}\}}$. 
Using these estimates, it constructs $2d$ candidates near the
estimated mean and tests them with the separation oracle of $\calA$.
If a candidate lies outside $\calA$, a separating inequality is added
to $D_{t,j}$ and the mean and covariance are recomputed with the updated $D_{t,j+1}$. 
 
Let $s_t$ be the number of cuts added in round $t$ and
$S_T=\sum_{t=1}^T s_t$. There are $T+S_T$ moment-estimation stages
and at most $2d$ separation queries per stage, so the number of calls to the separation oracles of $\calA$ is bounded by $2d(T+S_T)$. 
Section~5.2 of \citet{ito2019oracle} establishes $S_T=\calO(T)$, so the number of separation oracle calls is of order $\calO(dT)$.
 
By \citet[Theorem~5]{ito2019oracle}, the required accuracy for \cref{eq: estimation error 1} is
$\varepsilon=1/(12eT)$. The moment
estimation procedure described above therefore uses
$\widetilde{\calO}(d^3T^2)$ unnormalized density evaluations per
stage, including rounding under the stated geometric assumptions. Recall that this unnormalized density estimation is to evaluate $\exp(\eta \Thetahat_t a)\mathbf{1}_{\{a\in D_{t,j}\}}$, where the indicator $\mathbf{1}_{\{a\in D_{t,j}\}}$ involves the membership query to $D_{t,j}$ but not the original action set $\calA$. Hence, this unnormalized density estimation is counted as arithmetic operation but not oracle calls. At state $j\in[s_t]$ in round $t$, an evaluation of $\mathbf{1}_{\{a\in D_{t,j}\}}$ requires testing the
$S_{t-1}+j$ stored inequalities, which
costs $\calO(d^2+d(S_{t-1}+j))$ arithmetic operations. Summing over all rounds and all stages gives
\begin{equation}\label{eq:complexity-ito-arithmetic-retained-cuts}
\begin{aligned}
\widetilde{\calO}\!\left(
  d^3T^2\sum_{t=1}^T\sum_{j=0}^{s_t}
  [d^2+d(S_{t-1}+j)]\right)=\widetilde{\calO}\!\left(
 d^4(d+S_T)(T+S_T)T^2\right).
\end{aligned}
\end{equation}
Hence, for $S_T=\calO(T)$, the total number of arithmetic operations is $\widetilde{\calO}(d^5T^3 + d^4 T^4)$. 
\end{document}